\documentclass[letterpaper]{article} 
\usepackage[preprint]{aaai2027}  
\usepackage[hyphens]{url}  
\usepackage{graphicx} 
\usepackage{natbib}  
\usepackage{caption} 
\usepackage{algorithm}
\usepackage{algorithmic}
\usepackage{booktabs}

\usepackage{amsmath,amssymb,amsthm}
\usepackage{xcolor}
\usepackage{multirow}
\usepackage{colortbl}
\usepackage{subcaption}
\usepackage{enumitem}
\usepackage{mdframed}
\usepackage{pgf}
\usepackage{tikz}
\usepackage{array}
\usepackage{tabularx}
\usepackage{appendix}

\usetikzlibrary{shapes,arrows.meta,positioning,fit,backgrounds,calc,decorations.pathreplacing}

\definecolor{NavyBlue}{RGB}{13,43,94}
\definecolor{AccentBlue}{RGB}{26,87,153}
\definecolor{AccentGreen}{RGB}{21,87,36}
\definecolor{AccentOrange}{RGB}{190,80,0}
\definecolor{AccentRed}{RGB}{170,0,0}
\definecolor{LightBlue}{RGB}{213,232,245}
\definecolor{LightGreen}{RGB}{212,237,218}
\definecolor{LightOrange}{RGB}{255,243,224}
\definecolor{BoxGray}{RGB}{246,246,250}
\definecolor{TrtColor}{RGB}{220,53,69}
\definecolor{OutColor}{RGB}{40,167,69}
\definecolor{VitColor}{RGB}{0,123,255}
\definecolor{BRColor}{RGB}{102,16,242}
\definecolor{HeadColor}{RGB}{255,193,7}
\theoremstyle{plain}
\newtheorem{theorem}{Theorem}
\newtheorem{proposition}{Proposition}

\theoremstyle{definition}
\newtheorem{definition}{Definition}
\theoremstyle{remark}
\newtheorem{remark}{Remark}

\newcommand{\BR}{\mathrm{BR}}
\newcommand{\MI}{\mathcal{I}}

\newcommand{\E}{\mathbb{E}}
\newcommand{\R}{\mathbb{R}}

\title{Causal State-Space Model for Causal Inference: Estimating Longitudinal Individual Treatment Effects}
\author{Abisoye Abidakun\textsuperscript{\rm 1}, Mingjun Zhong\textsuperscript{\rm 1}, Georgios Leontidis\textsuperscript{\rm 2}}
\affiliations{
    \textsuperscript{\rm 1}Department of Computing Science, University of Aberdeen, Aberdeen, UK\\
    \textsuperscript{\rm 2}UiT The Arctic University of Norway\\
    a.abidakun.22@abdn.ac.uk, mingjun.zhong@abdn.ac.uk, georgios.leontidis@uit.no
}

\begin{document}

\maketitle

\begin{abstract}
Estimating counterfactual outcomes over time from longitudinal observational
data is central to clinical decision support.
Existing state-of-the-art methods rely on \emph{domain confusion}:
adversarial training that renders learned representations invariant to treatment assignment.
Yet this invariance creates a fundamental \emph{mutual information conflict},
suppressing treatment-correlated covariate signals that are independently
necessary for accurate outcome prediction.
We formalise this tension by deriving a Jensen--Shannon divergence bound on the
resulting counterfactual prediction error, and develop two complementary models
to resolve it.
First, \textbf{CSSD} (\emph{Causal State-Space model with Direct decoder}) adapts
selective State Space Models with a parallel multi-step decoder that eliminates
accumulated rollout error by producing all prediction horizons simultaneously in
a single forward pass at $O(T)$ encoder cost.
Second, \textbf{CSSPD} (\emph{Causal State-Space model with Predictive
regularisation and Direct decoder}) augments CSSD with Contrastive Predictive
Coding (CPC) and Local Information Maximisation (LIM): CPC reinforces temporal
predictability in the balancing representation, while LIM recovers the local
covariate information that domain confusion destroys.
On MIMIC-III, CSSPD achieves lower counterfactual RMSE than the Causal
Transformer at every horizon $\tau \geq 2$, with gains growing from 0.02
(2-step) to 0.07 (6-step).
On the Cancer Simulation benchmark across confounding strengths
$\gamma \in \{0,1,2,3,4\}$, CSSPD outperforms CT at $\gamma \leq 3$
(margins 25.9\%--37.0\%), while CSSD achieves the lowest overall average RMSE
(12.7\% reduction over CT), confirming the MI conflict analysis.
To our knowledge, this is the first work to formalise the
balancing--prediction MI conflict and propose a structured resolution through
complementary predictive and information-theoretic training objectives.
\end{abstract}

\section{Introduction}
\label{sec:intro}

Estimating \emph{individualized treatment effects} (ITEs) from longitudinal
observational data is a central challenge in clinical decision support
\citep{hernan2020causal,prosperi2020causal}.
Given a patient history $H_t = (A_{1:t-1}, Y_{1:t-1}, X_{1:t-1})$ (i.e., treatments, outcomes, and covariates), the
goal is to estimate $\E[Y_{t+\tau}(\bar{a}) \mid H_t]$ — the expected
potential outcome under a hypothetical treatment sequence $\bar{a}$ — under
\emph{time-dependent confounding}, where treatment assignments correlate with
both patient state and future outcomes
\citep{robins1986new,robins2000marginal}. Deep learning methods have made substantial progress; frameworks such as RMSN \citep{lim2018forecasting}, CRN \citep{bica2020estimating}, G-Net
\citep{li2021gnet}, and the Causal Transformer (CT) \citep{melnychuk2022causal}
each build on sequential models with domain confusion or reweighting to achieve
treatment-invariant representations.

Two structural limitations motivate the present work.
\textbf{Limitation 1 — The balancing-prediction tension.}
Domain confusion introduces an inherent mutual information conflict
(formalised in the \emph{theory} section).
Suppressing the mutual information between the balancing representation (a learned entity encoding up to time t, containing covariates and outcome history) and treatment, i.e., $I(\BR_t;\,A_t)$, forces the encoder to discard features of
$H_t$ that are correlated with treatment $A_t$, but many such features,
notably current vital signs $X_t$, are also independently predictive of
future outcomes $Y_{t+\tau}$.
Domain confusion therefore systematically removes the representational
information the decoder needs, creating a bottleneck that limits
counterfactual accuracy regardless of model capacity.
Recent empirical analysis \citep{huang2024empirical} confirms this
covariate information loss in CRN and CT, corroborating the theoretical
concern we formalise here.

\textbf{Limitation 2 — Autoregressive rollout error.}
Standard multi-step decoders predict future outcomes one step at a time,
feeding each prediction back as input for the next.
This autoregressive rollout compounds errors geometrically: if a single-step
prediction carries error $\varepsilon$, the $\tau$-step prediction carries
error $O(\varepsilon^\tau)$ \citep{taieb2014machine}.
At the prediction horizons relevant to clinical decision support ($\tau \geq 2$),
this rollout error accumulates rapidly and limits counterfactual accuracy
independently of representation quality.

To address both limitations simultaneously, this paper develops \textbf{CSSPD}. Our target is to make multi-step-ahead counterfactual outcome prediction, that is, given
$H_t$ and a hypothetical future treatment sequence
$\bar{a}_{t+1:t+\tau_{\max}}$, we want to produce simultaneous estimates
$\hat{Y}_{t+1}, \ldots, \hat{Y}_{t+\tau_{\max}}$ for all horizons in a single
forward pass, without autoregressive rollout. CSSPD introduces two core innovations: the encoder and the decoder.
For the \textbf{encoder}, we adapt the 
selective State Space Models \citep{gu2023mamba} that process each input stream
in $O(T)$ time while retaining the capacity to selectively integrate long-range
history. This should be compared to the $O(T^2)$ cost of CT (\citet{melnychuk2022causal}).
A \texttt{CausalGatedMixer} fuses the treatment, outcome, and covariate streams
through scalar gates whose initialisations encode the structural causal graph,
producing a Balancing Representation ($\BR_t$) that satisfies the backdoor
criterion \citep{pearl2009causality}.
A parallel multi-step \textbf{decoder} then produces all $\tau_{\max}$ counterfactual
predictions simultaneously, eliminating the $O(\epsilon^\tau)$ error
accumulation of autoregressive rollout \citep{taieb2014machine}.
For the \textbf{training objectives}, CSSPD augments domain confusion with two
auxiliary regularisers that resolve the MI conflict - a CPC head that preserves temporal predictability in $\BR_t$, and a LIM head that recovers covariate information that domain
confusion would otherwise suppress, each targeting a distinct facet of the
representational bottleneck without compromising treatment invariance. \textbf{Our contributions are:}
\begin{enumerate}[itemsep=3pt]

\item  We proposed a causal state-space model with direct encoder (CSSD) for causal inference - a selective SSM encoder with a
      parallel multi-step decoder that produces $\tau$ counterfactual
      predictions in a single forward pass, eliminating rollout error
      accumulation.
\item We further proposed a model named Causal State-Space model with Predictive regularisation and Direct decoder (CSSPD), which is a CSSD augmented with CPC and LIM objectives. This model resolves the MI conflict without sacrificing treatment invariance.
\item We formalised the \emph{balancing-prediction MI conflict} and derived
      a Jensen--Shannon bound on the counterfactual prediction error induced
      by domain confusion. 
\item We conducted empirical evaluation on the real-world data set MIMIC-III and the synthetic data set Cancer Simulation, showing CSSPD reduces counterfactual RMSE at all horizons $\tau \geq 2$ on MIMIC-III (growing margin 0.02--0.07) and outperforms CT at $\gamma\!\leq\!3$ on Cancer Simulation (margins 25.9\%--37.0\%), with gains consistent with the MI conflict theory.
\end{enumerate}



\section{Problem Settings}
\label{sec:background}

For simplicity, we use the MIMIC-III data set to set up the problem, but this should apply to other longitudinal data in general. Let $i$ index a patient with longitudinal record over $t = 1,\ldots,T^{(i)}$.
At each step, the record contains $p$ time-varying covariates
$X_t^{(i)} \in \mathbb{R}^p$, binary treatment indicators
$A_t^{(i)}\in\{0,1\}^2$, scalar outcome $Y_t^{(i)}\in\mathbb{R}$,
and time-invariant baseline attributes $S^{(i)}\in\mathbb{R}^q$.
The patient history is
$H_t = \{\hat{X}_t,\, \hat{A}_{t-1},\, \hat{Y}_t,\, S\}$,
where $\hat{X}_t = (X_1, \ldots, X_t)$, $\hat{Y}_t = (Y_1, \ldots, Y_t)$,
$\hat{A}_{t-1} = (A_1, \ldots, A_{t-1})$.
The \emph{counterfactual prediction task} is: given $H_t$, estimate
\begin{equation}
\E\bigl[Y_{t+\tau}(\bar{a}) \mid H_t\bigr],\quad
\tau \in \{1, \ldots, \tau_{\max}\},
\label{eq:task}
\end{equation}
for hypothetical treatment sequences $\bar{a}$, under the potential outcome
framework \citep{rubin2005causal}. Sequential ignorability and domain confusion are required for causal inference.

\textbf{Sequential ignorability} is defined as
$Y_t(\bar{a}) \perp\!\!\!\perp A_t \mid H_t$ \citep{robins1986new}.
Under sequential ignorability, counterfactual outcomes are identifiable
if positivity holds
(i.e.\ $P(A_t = a \mid H_t) > 0$ for all treatment values $a$,
ensuring every treatment is observed across all patient histories)
and the balancing representation $\BR_t = \phi(H_t)$ satisfies the
backdoor criterion \citep{pearl2009causality}.

\textbf{Domain confusion} is the adversarial mechanism used by
\citet{bica2020estimating} to enforce treatment invariance in $\BR_t$,
thereby satisfying both conditions above.
The goal is to make treatment $A_t$ unpredictable from $\BR_t$:
if the representation carries no information about which treatment
a patient received, it cannot be confounded by treatment selection.
To achieve this, a discriminator $D$ is trained to predict treatment
from $\BR_t$, minimising the cross-entropy loss
\begin{equation}
\mathcal{L}_\text{DC} = -\E_t\bigl[\log p(A_t \mid \BR_t)\bigr].
\label{eq:dc}
\end{equation}
The encoder $\phi$ is trained adversarially to maximise the same loss,
that is, to actively prevent $D$ from succeeding.
This is implemented via a gradient reversal layer (GRL)
\citep{ganin2016domain}, which negates the gradient signal flowing back
to $\phi$, making the effective encoder update $-\alpha\,\mathcal{L}_\text{DC}$.
The result is a representation $\BR_t$ that is invariant to treatment
assignment, satisfying the backdoor criterion.
It is precisely this enforced invariance, however, that introduces the
mutual information conflict we formalise in the \emph{theory section}.
\section{Model Architecture}
\label{sec:architecture}

The proposed CSSPD architecture is shown in Figure~\ref{fig:architecture}.
Three input streams; treatments ($\hat{A_t}$), outcomes ($\hat{Y_t}$), and
covariates ($\hat{X_t}$), are encoded by independent SSM layers and fused by
a \texttt{CausalGatedMixer}.
The resulting balancing representation $\BR_t$ is fed to a parallel multi-step
decoder and CPC/LIM regularisation heads that resolve the MI conflict
introduced by domain confusion \emph{(theory section)}.

CSSPD combines two distinct modifications to the base encoder: a structural
change to the treatment stream's sequence operator, and the CPC and LIM
contrastive objectives.
To quantify the independent contribution of each, we introduce two additional
variants.
\textbf{CHSD} (\emph{Causal Hybrid State-Space Decoder}) applies only the
structural change, replacing the treatment-stream SSM with the standard multi-head self-attention  of \citet{melnychuk2022causal}, giving the encoder lossless
access to the full treatment history while retaining SSMs for the outcome and covariate streams.
\textbf{CHSPD} (\emph{Causal Hybrid State-Space model with Predictive
regularisation and Direct decoder}) combines both modifications, augmenting
CHSD with the CPC and LIM objectives from CSSPD with results details in the \emph{experiment section}.



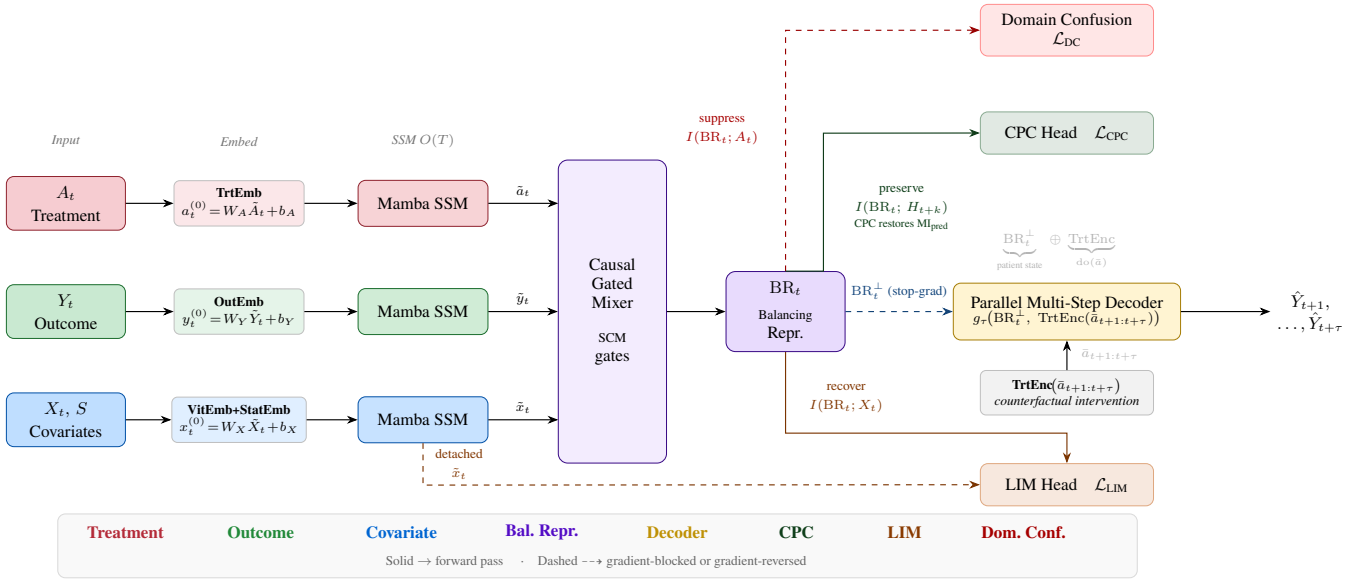
\begin{figure*}[t]
\centering
\resizebox{\textwidth}{!}{%
\begin{tikzpicture}[
  box/.style={draw, rounded corners=4pt,
              minimum width=2.2cm, minimum height=0.85cm,
              font=\small, align=center, inner sep=5pt},
  emb/.style={draw=gray!50, fill=gray!10, rounded corners=3pt,
              minimum width=2.4cm, minimum height=0.65cm,
              font=\scriptsize, align=center, inner sep=4pt},
  ssm/.style={box, minimum width=2.4cm, minimum height=0.85cm},
  head/.style={draw, rounded corners=4pt,
               minimum width=3.2cm, minimum height=0.78cm,
               font=\footnotesize, align=center, inner sep=5pt},
  brbox/.style={draw=BRColor!55!black, fill=BRColor!15, rounded corners=5pt,
                minimum width=2.2cm, minimum height=1.25cm,
                font=\small, align=center, inner sep=5pt},
  arr/.style={->, >=Stealth, semithick},
  darr/.style={->, >=Stealth, semithick, dashed},
  lbl/.style={font=\scriptsize, align=center},
]

\def\xA{0.0}
\def\xB{3.2}
\def\xC{6.6}
\def\xD{10.1}  
\def\xE{13.3}  
\def\xF{18.5}  
\def\xG{23.0}  

\def\yT{4.0}   
\def\yO{2.0}   
\def\yC{0.0}   

\node[lbl, black!55] at (\xA, 5.15) {\textit{Input}};
\node[lbl, black!55] at (\xB, 5.15) {\textit{Embed}};
\node[lbl, black!55] at (\xC, 5.15) {\textit{SSM~$O(T)$}};

\node[box, fill=TrtColor!25, draw=TrtColor!65!black] (At)
  at (\xA,\yT) {$A_t$\\[2pt]\footnotesize Treatment};
\node[box, fill=OutColor!25, draw=OutColor!65!black] (Yt)
  at (\xA,\yO) {$Y_t$\\[2pt]\footnotesize Outcome};
\node[box, fill=VitColor!25, draw=VitColor!65!black] (Xt)
  at (\xA,\yC) {$X_t,\,S$\\[2pt]\footnotesize Covariates};

\node[emb, fill=TrtColor!12] (AEmb) at (\xB,\yT)
  {\textbf{TrtEmb}\\[-1pt]$a_t^{(0)}\!=\!W_A\tilde{A}_t\!+\!b_A$};
\node[emb, fill=OutColor!12] (YEmb) at (\xB,\yO)
  {\textbf{OutEmb}\\[-1pt]$y_t^{(0)}\!=\!W_Y\tilde{Y}_t\!+\!b_Y$};
\node[emb, fill=VitColor!12] (XEmb) at (\xB,\yC)
  {\textbf{VitEmb+StatEmb}\\[-1pt]$x_t^{(0)}\!=\!W_X\tilde{X}_t\!+\!b_X$};

\draw[arr] (At) -- (AEmb);
\draw[arr] (Yt) -- (YEmb);
\draw[arr] (Xt) -- (XEmb);

\node[ssm, fill=TrtColor!22, draw=TrtColor!65!black] (SSMa)
  at (\xC,\yT) {Mamba SSM};
\node[ssm, fill=OutColor!22, draw=OutColor!65!black] (SSMy)
  at (\xC,\yO) {Mamba SSM};
\node[ssm, fill=VitColor!22, draw=VitColor!65!black] (SSMx)
  at (\xC,\yC) {Mamba SSM};

\draw[arr] (AEmb) -- (SSMa);
\draw[arr] (YEmb) -- (SSMy);
\draw[arr] (XEmb) -- (SSMx);

\node[draw=BRColor!55!black, fill=BRColor!8, rounded corners=5pt,
      minimum width=2.0cm, minimum height=5.6cm,
      font=\small, align=center, inner sep=6pt]
  (Mixer) at (\xD,\yO)
  {Causal\\Gated\\Mixer\\[8pt]\scriptsize SCM\\gates};

\draw[arr] (SSMa.east) -- node[above,lbl]{$\tilde{a}_t$} (9.1,\yT);
\draw[arr] (SSMy.east) -- node[above,lbl]{$\tilde{y}_t$} (Mixer.west);
\draw[arr] (SSMx.east) -- node[above,lbl]{$\tilde{x}_t$} (9.1,\yC);

\node[brbox] (BR) at (\xE,\yO)
  {$\BR_t$\\[4pt]\scriptsize Balancing\\Repr.};
\draw[arr] (Mixer.east) -- (BR.west);


\node[head, fill=red!10, draw=red!45] (DC)
  at (\xF, 7.2) {Domain Confusion\\$\mathcal{L}_\text{DC}$};

\node[head, fill=AccentGreen!13, draw=AccentGreen!50] (CPC)
  at (\xF, 5.3) {CPC Head\quad$\mathcal{L}_\text{CPC}$};

\node[draw=HeadColor!65!black, fill=HeadColor!18, rounded corners=4pt,
      minimum width=4.2cm, minimum height=1.05cm,
      font=\small, align=center, inner sep=5pt] (Dec)
  at (\xF,\yO)
  {Parallel Multi-Step Decoder\\[-3pt]\scriptsize
   $g_\tau\!\bigl(\BR_t^\perp,\;\mathrm{TrtEnc}(\bar{a}_{t+1:t+\tau})\bigr)$};

\node[lbl, gray!58, anchor=south] at (18.3, \yO+0.6)
  {\scriptsize$\underbrace{\BR_t^\perp}_{\text{patient state}}\;\oplus\;
        \underbrace{\mathrm{TrtEnc}}_{\mathrm{do}(\bar{a})}$};

\node[emb, minimum width=3.2cm] (TrtEnc) at (\xF, 0.5)
  {\textbf{TrtEnc}$\!\bigl(\bar{a}_{t+1:t+\tau}\bigr)$\\[-2pt]
   \scriptsize\textit{counterfactual intervention}};

\node[head, fill=AccentOrange!13, draw=AccentOrange!50] (LIM)
  at (\xF, -1.2) {LIM Head\quad$\mathcal{L}_\text{LIM}$};


\draw[darr, red!60!black]
  (BR.north) -- ++(0, 4.45) -- (DC.west);
\node[lbl, red!65!black, anchor=east] at (12.92, 5.35)
  {suppress\\$I(\BR_t;A_t)$};

\draw[arr, AccentGreen!68!black]
  (BR.north) -- ++(0.7, 0) -- ++(0, 2.55) -- (CPC.west);
\node[lbl, AccentGreen!72!black, anchor=west] at (14.45, 3.90)
  {preserve\\$I(\BR_t;\,H_{t+k})$\\[-2pt]{\tiny CPC restores MI$_\text{pred}$}};

\draw[darr, AccentBlue!78!black]
  (BR.east)
  -- node[above=3pt, lbl, AccentBlue!82!black]
       {$\BR_t^\perp$\;\scriptsize(stop-grad)}
  (Dec.west);

\draw[arr]
  (TrtEnc.north)
  -- node[right=4pt, lbl, gray!58]{$\bar{a}_{t+1:t+\tau}$}
  (Dec.south);

\draw[darr, AccentOrange!65!black]
  (SSMx.south)
  -- node[right=3pt, lbl, AccentOrange!68!black]
       {\scriptsize detached\\\scriptsize$\tilde{x}_t$}
  ++(0, -0.775) -- (LIM.west);

\draw[arr, AccentOrange!65!black]
  (BR.south) -- ++(0, -1.5) -| (LIM.north);
\node[lbl, AccentOrange!68!black, anchor=west] at (13.62, 0.40)
  {recover\\$I(\BR_t;X_t)$};

\node[font=\small, align=center] (Ypred) at (\xG,\yO)
  {$\hat{Y}_{t+1},$\\$\ldots,\hat{Y}_{t+\tau}$};
\draw[arr] (Dec.east) -- (Ypred.west);

\begin{scope}[shift={(0,-2.30)}]
  \draw[rounded corners=4pt, gray!30, fill=gray!5]
    (-0.15,-0.62) rectangle (19.8, 0.56);
  \node[lbl, TrtColor!82!black,     font=\footnotesize] at (1.1,  0.22) {\textbf{Treatment}};
  \node[lbl, OutColor!82!black,     font=\footnotesize] at (3.6,  0.22) {\textbf{Outcome}};
  \node[lbl, VitColor!82!black,     font=\footnotesize] at (6.2,  0.22) {\textbf{Covariate}};
  \node[lbl, BRColor!82!black,      font=\footnotesize] at (8.8,  0.22) {\textbf{Bal.\ Repr.}};
  \node[lbl, HeadColor!75!black,    font=\footnotesize] at (11.3, 0.22) {\textbf{Decoder}};
  \node[lbl, AccentGreen!72!black,  font=\footnotesize] at (13.5, 0.22) {\textbf{CPC}};
  \node[lbl, AccentOrange!72!black, font=\footnotesize] at (15.5, 0.22) {\textbf{LIM}};
  \node[lbl, red!65!black,          font=\footnotesize] at (17.7, 0.22) {\textbf{Dom.\ Conf.}};
  \node[lbl, black!60] at (9.8, -0.32)
    {Solid $\rightarrow$ forward pass
     \quad$\cdot$\quad
     Dashed $\dashrightarrow$ gradient-blocked or gradient-reversed};
\end{scope}

\end{tikzpicture}%
}%
\caption{%
  \textbf{CSSPD architecture.}
  Three streams ($A_t$, $Y_t$, $X_t/S$) are embedded and passed through $L$
  independent Mamba SSM layers ($O(T)$) to produce
  $\tilde{a}_t,\tilde{y}_t,\tilde{x}_t$.
  A SCM-gated \texttt{CausalGatedMixer} (\ref{eq:mixer}) fuses them into the
  balancing representation $\BR_t$.
  \emph{Domain confusion} (dashed red) suppresses $I(\BR_t;A_t)$ via gradient
  reversal but inadvertently suppresses $I(\BR_t;H_{t+k})$ and $I(\BR_t;X_t)$.
  \emph{CPC} (solid green) restores temporal MI $I(\BR_t;H_{t+k})$;
  \emph{LIM} (solid orange) restores covariate MI $I(\BR_t;X_t)$.
  The parallel decoder (\ref{eq:decoder}) receives two distinct inputs:
  a \emph{stop-gradient} copy $\BR_t^\perp$ (dashed blue, left-entering)
  encoding patient state, and $\mathrm{TrtEnc}(\bar{a}_{t+1:t+\tau})$
  (solid black, bottom-entering) encoding the counterfactual intervention;
  together they produce all $\tau_{\max}$ predictions simultaneously.
  CSSD omits the CPC and LIM heads.
}
\label{fig:architecture}
\end{figure*}

\subsection{The Input Encoding}

Each stream at time step $t$ is projected to $\R^{d_\text{model}}$ via a
stream-specific linear embedding (e.g., \textbf{TrtEmb}, \textbf{OutEmb},
and \textbf{VitEmb+StatEmb} in Figure~\ref{fig:architecture}).
 We concatenate the \emph{planned} next
treatment $A_{t+1}^*$ (observed at time $t$ as the clinician's intended action
for the next step) with the actual current treatment $A_t$ to inform the encoder
of the intended treatment trajectory,  and denote the concatenated matrix $\tilde{A}_t=[A_t;A_{t+1}^*]$. Let $d_{model}$ denote the common embedding dimension shared across all three streams, the embeddings are:
\begin{align}
{\textbf{TrtEmb:}}\quad a_t^{(0)} &= W_A \tilde{A}_t + b_A,   \label{eq:trt-emb}\\
\textbf{OutEmb:}\quad y_t^{(0)} &= W_Y\tilde{Y}_t  + b_Y,   \label{eq:out-emb}\\
\textbf{VitEmb+StatEmb:}\quad x_t^{(0)} &= W_X\tilde{X}_t + b_X,            \label{eq:cov-emb}
\end{align}
where $\tilde{Y}_t=[Y_t;\, X_t]$, $\tilde{X}_t=[X_t;\, S]$, $S$ is the static covariate vector broadcast to every timestep,
and $W_A\in\mathbb{R}^{d_{model}\times 2},\;
     W_Y\in\mathbb{R}^{d_{model}\times (p+1)},\;
     W_X\in\mathbb{R}^{d_{model}\times (p+q)}$ are learnable projection
matrices and $b_A, b_Y, b_X\in\mathbb{R}^{d_{model}}$ are learnable bias vectors.
\subsection{The SSM Encoder (CausalMambaLayer)}
\label{sec:ssm-encoder} Embeddings from (\ref{eq:trt-emb})--(\ref{eq:cov-emb}) are passed independently
through $L$ stacked \texttt{CausalMambaLayer} blocks, one per stream
$u\in\{A,Y,X\}$. Set $s_{t,A}^{(0)} = a_t^{(0)},  s_{t,Y}^{(0)} = y_t^{(0)},s_{t,X}^{(0)} = x_t^{(0)}$.
Each block implements a \emph{selective} SSM \citep{gu2023mamba} with
input-dependent transition matrices, applied identically to every stream
$u \in \{A,\,Y,\,X\}$:
\begin{align}
h_{t,u}^{(l)} &= \bar{A}_t\!\left(s_{t,u}^{(l)}\right) h_{t-1,u}^{(l)}
               + \bar{B}_t\!\left(s_{t,u}^{(l)}\right) s_{t,u}^{(l)},
  \label{eq:ssm-state}\\
s_{t,u}^{(l+1)} &= s_{t,u}^{(l)} + C_t\!\left(s_{t,u}^{(l)}\right) h_{t,u}^{(l)}
                 + D\, s_{t,u}^{(l)},
  \label{eq:ssm-out}
\end{align}
where $\bar{A}_t, \bar{B}_t, C_t$ are input-dependent matrices,
$D$ is a fixed scalar skip connection that lets the current input
pass through directly, and \eqref{eq:ssm-out} adds a residual connection.
After $L$ layers:
\begin{equation}
\tilde{a}_t = s_{t,A}^{(L)},\quad \tilde{y}_t = s_{t,Y}^{(L)},\quad
\tilde{x}_t = s_{t,X}^{(L)} \in\R^{d_\text{model}}.
\label{eq:ssm-outputs}
\end{equation}

\subsection{The Causal Gated Mixer}
\label{sec:mixer}

Given $\tilde{a}_t, \tilde{y}_t, \tilde{x}_t$ from (\ref{eq:ssm-outputs}),
the \texttt{CausalGatedMixer} combines them via learned scalar gates whose
initialisations encode the structural causal model (SCM):
\begin{equation}
\BR_t = \sigma(g_{A \to Y})\,\tilde{a}_t \oplus
        \sigma(g_{Y \to A})\,\tilde{y}_t \oplus \tilde{x}_t,
\label{eq:mixer}
\end{equation}
where $g_{A \to Y}$ is initialised to $1$ ($\sigma(1) \approx 0.73$,
near-open, reflecting that treatment causally influences outcomes in the
SCM) and $g_{Y \to A}$ is initialised to $-3$ ($\sigma(-3) \approx 0.05$,
near-closed, reflecting that the reverse direction $Y \to A$ is not
structurally causal).
Gate initialisations for all streams are shown in
Figure~S1 of the Supplementary Material.
All gates remain learnable throughout training.Here $\oplus$ denotes concatenation followed by a linear projection
$W \in \mathbb{R}^{d_{BR} \times 3d_\text{model}}$, producing
$\BR_t \in \mathbb{R}^{d_{BR}}$.
The pre-mixer covariate embedding $\tilde{x}_t$ is also fed (detached) to
the LIM head to recover local covariate information suppressed by domain
confusion.

\subsection{The Parallel Multi-Step Decoder}

It is well known that autoregressive rollout accumulates errors as $O(\epsilon^\tau)$ with horizon
\citep{taieb2014machine} when modelling longitudinal data. Instead, we propose the \texttt{ParallelMultiStepDecoder}, which replaces autoregressive
rollout with $\tau_{\max}$ independent prediction heads, each conditioned
on a stop-gradient copy of $\BR_t$ and a
hypothetical treatment-sequence embedding:
\begin{equation}
\hat{Y}_{t+\tau}(\bar{a}) = g_\tau\!\left(\BR_t^\perp,\,
  \mathrm{TrtEnc}(\bar{a}_{t+1:t+\tau})\right), \tau = 1, \ldots, \tau_{\max},
\label{eq:decoder}
\end{equation}
where $\BR_t^\perp = \mathrm{stop\text{-}grad}(\BR_t)$, and
$\mathrm{TrtEnc}(\bar{a}_{t+1:t+\tau})$ is defined as a shared per-step MLP
$\psi\!:\!\mathbb{R}^{d_A}\!\to\!\mathbb{R}^{d_\mathrm{trt}}$
(Linear\,--\,GELU\,--\,LayerNorm) applied independently to each future treatment
step.
At horizon $\tau$, TrtEnc produces two embeddings:
a \emph{cumulative} term
$e^\mathrm{hist}_\tau = \tfrac{1}{\tau}\sum_{k=1}^{\tau}\psi(a_{t+k})$,
capturing the running treatment context up to step $\tau$ (e.g.\
pharmacokinetic accumulation), and a \emph{step-specific} term
$e^\mathrm{cur}_\tau = \psi(a_{t+\tau})$,
capturing the treatment at exactly step $\tau$.
This two-component design resolves ordering ambiguity: sequences that share the
same cumulative mean but differ in ordering (e.g.\ high-dose then none vs.\
none then high-dose) produce identical $e^\mathrm{hist}$ yet distinct
$e^\mathrm{cur}$, making TrtEnc the counterfactual input that separates
different treatment regimes at inference time.

Each $g_\tau$ is an independent MLP so that short- and long-horizon
predictions can learn different mappings; all heads share a common trunk
but diverge at the final linear layer.
Each head passes
$[\BR_t^\perp;\,e^\mathrm{hist}_\tau;\,e^\mathrm{cur}_\tau]$
through a two-layer GELU trunk and produces $\hat{Y}_{t+\tau}$.
Because every head reads the same fixed $\BR_t^\perp$, errors cannot accumulate
across horizons as they do in autoregressive rollout.
$\BR_t^\perp$ (the balanced representation) serves as a treatment-invariant
confounder summary satisfying the backdoor criterion, while
$\mathrm{TrtEnc}(\bar{a}_{t+1:t+\tau})$ encodes the hypothetical intervention,
supplying the $\mathrm{do}(\bar{a})$ operator independently of the observed history.
Together they implement the backdoor-adjusted estimand:
\begin{equation}
  \mathbb{E}\bigl[Y_{t+\tau}(\bar{a})\bigr]
    \;\approx\;
    g_\tau\!\Bigl(
      \underbrace{\BR_t^\perp}_{\text{patient state}},\;
      \underbrace{\mathrm{TrtEnc}(\bar{a}_{t+1:t+\tau})}_{\text{intervention}}
    \Bigr).
  \label{eq:causal-split}
\end{equation}

Equation~(\ref{eq:causal-split}) makes a causal rather than predictive
claim via the backdoor criterion: $\BR_t^\perp$ blocks confounding paths
while $\mathrm{TrtEnc}(\bar{a}_{t+1:t+\tau})$ encodes the hypothetical
intervention $\mathrm{do}(\bar{a})$ separately from the observed history.

\subsection{Contrastive Predictive Coding (CPC) and Local Information Maximisation (LIM) Heads}

Domain confusion drives $I(\BR_t;\,A_t)\!\to\!0$ but inadvertently
suppresses $I(\BR_t;\,Y_{t+\tau})$ (temporal predictive information) and
$I(\BR_t;\,X_t)$ (local covariate information), creating the MI conflict.
CPC and LIM are targeted regularisers operating on disjoint pathways that
restore these two quantities.

\paragraph{CPC.}
A \texttt{CPCHead} applies $K$ linear predictors $f_k$ to $\BR_t$,
each attempting to predict $\BR_{t+k}$ via InfoNCE over in-batch negatives
\citep{oord2018representation}.
InfoNCE lower-bounds $I(\BR_t;\,\BR_{t+k})$, so maximising it forces
$\BR_t$ to retain temporal structure.
The connection to $I(\BR_t;\,Y_{t+\tau})$ — the quantity suppressed by
domain confusion — follows from the \emph{Data Processing Inequality}
(DPI): for any deterministic function $f$,
$I(Z;\,f(X))\leq I(Z;\,X)$, since processing can only reduce information.
Applying DPI twice: \emph{(i)} since $\BR_{t+k}=\phi(H_{t+k})$ is a
deterministic function of future history $H_{t+k}$, we have
$I(\BR_t;\,\BR_{t+k})\leq I(\BR_t;\,H_{t+k})$; \emph{(ii)} since
$Y_{t+\tau}$ is a component of $H_{t+k}$ for $k\geq\tau$,
$I(\BR_t;\,Y_{t+\tau})\leq I(\BR_t;\,H_{t+k})$.
Both $\BR_{t+k}$ and $Y_{t+\tau}$ are therefore compressed views of the
same underlying history; maximising $I(\BR_t;\,\BR_{t+k})$ prevents
$\BR_t$ from discarding temporal structure in $H_{t+k}$, which supports
$I(\BR_t;\,Y_{t+\tau})$ and counteracts the domain confusion collapse.

A LIM head maximises $I(\tilde{x}_t;\,\BR_t)$ — a
tractable proxy for $I(X_t;\,\BR_t)$ — via InfoNCE, with $\tilde{x}_t$
detached so gradients flow into the encoder only through $\BR_t$,
recovering covariate information that CPC does not.

\subsection{The Training Objective function of CSSD}
\label{sec:sstd}

The objective function of our proposed model, named Causal State-Space model with Direct decoder (CSSD), combines a primary output prediction loss, domain confusion, and the parallel multi-step decoder loss:
\begin{equation}
\mathcal{L}_\text{CSSD}
  = \mathcal{L}_\text{pred}
  + \alpha\,\mathcal{L}_\text{DC}
  + \mathcal{L}_\text{MS}.
\label{eq:sstd-loss}
\end{equation}
Each of the loss function is defined as follows. The primary prediction loss is defined as:
\begin{equation}
\mathcal{L}_\text{pred}
  = \E\!\left[\bigl(Y_{t+1} - G^Y(\BR_t,\,A_{t+1})\bigr)^2\right],
\label{eq:pred-loss}
\end{equation}
the domain confusion loss is defined in the equation \ref{eq:dc}, and the parallel multi-step decoder loss is defined as
\begin{equation}
\mathcal{L}_\text{MS}
  = \lambda_\text{MS} \sum_{\tau=1}^{\tau_{\max}}
    \E\!\left[\bigl(Y_{t+\tau}(\bar{a}) -
    \hat{Y}_{t+\tau}(\bar{a})\bigr)^2\right].
\label{eq:ms-loss}
\end{equation}

Because no stop-gradient is applied, $\mathcal{L}_\text{pred}$ backpropagates
through the encoder $\phi$, directly training it to produce representations
that minimise one-step factual prediction error, and it is distinct from $\mathcal{L}_\text{MS}$, which uses the stop-gradient copy $\BR_t^\perp$ and therefore does \emph{not} train the encoder. $\alpha$ is the domain confusion weight annealed exponentially during
training.

\subsection{The Training Objective Functoin of CSSPD}
\label{sec:sstcpd}
We further augments CSSD with CPC and LIM, which results in our another model, called Causal State-Space model with Predictive regularisation and Direct decoder (CSSPD). Firstly, the CPC loss is defined by using the InfoNCE:
\begin{equation}
\mathcal{L}_\text{CPC} = -\sum_{k=1}^{K} \E\left[
  \log \frac{\exp(f_k(\BR_t)^\top \BR_{t+k})}
  {\sum_{j=1}^{N_\text{neg}} \exp(f_k(\BR_t)^\top \BR_{t+k}^{(j)})}
\right],
\label{eq:cpc}
\end{equation}
where $f_k$ is a learned predictor for horizon $k$ and
$\{\BR_{t+k}^{(j)}\}_{j=1}^{N_\text{neg}}$ are in-batch negative samples. Secondly, the LIM loss is defined as:
\begin{equation}
\mathcal{L}_\text{LIM} = -\E\left[
  \log \frac{\exp(\tilde{x}_t^\top \BR_t)}
  {\sum_{j=1}^{M} \exp(\tilde{x}_t^\top \BR_t^{(j)})}
\right],
\label{eq:lim}
\end{equation}
where $\tilde{x}_t$ is \emph{detached} from encoder gradients. Finally, the objective functoin of CSSPD follows
\begin{equation}
\mathcal{L}_\text{CSSPD}
  = \mathcal{L}_\text{pred}
  + \alpha\,\mathcal{L}_\text{DC}
  + \mathcal{L}_\text{MS}
  + \lambda_\text{CPC}\,\mathcal{L}_\text{CPC}
  + \lambda_\text{LIM}\,\mathcal{L}_\text{LIM}.
\label{eq:full-loss}
\end{equation}
CPC and LIM are warmed up for $E_\text{warm}$ epochs before activation,
allowing domain confusion to establish an initial treatment-invariant
representation before contrastive objectives begin regularising.

\section{Theoretical Analysis: The Balancing-Prediction MI Conflict} \label{sec:theory}
This section provides the theoretical foundation for the CSSPD design.
We establish three results.
Proposition~\ref{prop:mi-conflict} proves that the MI conflict is not a
modelling artefact but an information-theoretic inevitability: any encoder
trained with domain confusion is forced to discard outcome-predictive
information.
Theorem~\ref{thm:js-bound} then quantifies the resulting cost, bounding the
joint factual and counterfactual prediction error in terms of the
Jensen--Shannon divergence between treatment-group representations — making
the balancing--prediction trade-off mathematically precise.
Finally, Proposition~\ref{prop:cpc-lim} proves that the CPC and LIM
objectives provide tractable lower bounds on the two MI quantities suppressed
by domain confusion ($\MI_\text{pred}$ and $\MI_\text{loc}$), showing
formally how CSSPD resolves the conflict that domain confusion creates.
\subsection{The Mutual Information Conflict}
\label{sec:mi-conflict}

Domain confusion suppresses $I(\BR_t;\,A_t)$ by removing treatment-correlated
features from $\BR_t$.
Because covariates $X_t$ are correlated with both $A_t$ (observational data)
and $Y_{t+\tau}$ (outcomes), this inadvertently destroys outcome-predictive
information — the \emph{MI conflict}.
Let $\phi: \mathcal{H} \to \mathbb{R}^d$ maps history $H_t$ to
$\BR_t$.
Define three mutual information quantities:
\begin{align}
  \MI_{\text{bal}}  &= I(\BR_t;\, A_t), \label{eq:mi-bal-def} \\
  \MI_{\text{pred}} &= I(\BR_t;\, Y_{t+\tau}), \label{eq:mi-pred-def} \\
  \MI_{\text{loc}}  &= I(\BR_t;\, X_t). \label{eq:mi-loc-def}
\end{align}
\begin{definition}[MI Conflict]
\label{def:mi-conflict}
An encoder $\phi$ suffers the \emph{MI conflict} if there exists no
representation $\BR_t = \phi(H_t)$ that simultaneously satisfies
$I(\BR_t;\, A_t) = 0$ (perfect balancing) and
$I(\BR_t;\, Y_{t+\tau}) = I(H_t;\, Y_{t+\tau})$ (no loss of
outcome-predictive information).
The conflict is non-vacuous whenever
$I(X_t;\, Y_{t+\tau} \mid A_t) > 0$, i.e.\ covariates carry
outcome-predictive information beyond what treatment assignment alone
reveals.
\end{definition}

In the following, we show that Theorem~\ref{thm:js-bound} quantifies this trade-off; CSSPD resolves it via
CPC (maximising $\MI_{\text{pred}}$) and LIM (maximising $\MI_{\text{loc}}$)
(Proposition~\ref{prop:cpc-lim}).

\subsection{Formalisation}

\begin{proposition}[MI Conflict under Domain Confusion]
\label{prop:mi-conflict}
Suppose $X_t \not\!\perp\!\!\!\perp A_t$ (confounding: covariates and treatment are
correlated) and $X_t \to Y_{t+\tau}$ (covariates predict outcomes).
Then any encoder $\phi$ that reduces $\MI_\text{bal}$ via domain confusion
also reduces $\MI_\text{loc}$, and consequently reduces the upper bound on
$\MI_\text{pred}$.
\end{proposition}

\begin{proof}[Proof sketch]
Domain confusion minimises $\MI_\text{bal} = I(\BR_t; A_t)$ by
suppressing all features of $H_t$ that are correlated with $A_t$.
Since $X_t \not\!\perp\!\!\!\perp A_t$, $X_t$ carries information about $A_t$;
therefore suppressing $I(\BR_t; A_t)$ also reduces $I(\BR_t; X_t) = \MI_\text{loc}$.
Since $X_t \to Y_{t+\tau}$, applying the chain rule of mutual information:
\[
I(\BR_t;\, Y_{t+\tau})
  \;\leq\; I(\BR_t;\, X_t) + I(\BR_t;\, Y_{t+\tau} \mid X_t).
\]
Under the assumption that $X_t$ is the dominant confounder (i.e.\ the
residual term $I(\BR_t;\,Y_{t+\tau}\mid X_t)$ is small relative to
$I(\BR_t;\,X_t)$), reducing $\MI_\text{loc}$ consequently reduces the
achievable $\MI_\text{pred}$.
This assumption is explicit, the proposition holds conditionally on the
regime where $X_t$ is the primary mediator from $H_t$ to $Y_{t+\tau}$,
which is precisely the clinical setting of interest (current vital signs
dominate short-horizon outcome prediction).
\end{proof}

\subsection{Jensen-Shannon Divergence Bound on Prediction Error}

Let $P_j^\phi = P^\phi(\BR_t \mid A_{t+1} = a^{(j)})$ denote the
balancing representation distribution conditioned on treatment regime $j$.

\begin{theorem}[Prediction Error Bound]
\label{thm:js-bound}
Let $\phi$ be any encoder, $G^Y$ be any outcome hypothesis, and let
$\pi^{(0)}, \pi^{(1)} > 0$ with $\pi^{(0)} + \pi^{(1)} = 1$.
Define $W = 2S / (\sqrt{\pi^{(0)}}\pi^{(1)} + \sqrt{\pi^{(1)}}\pi^{(0)})$
where $S = \sup_{H,Y} \E_{Y \mid H}[\mathcal{L}(G^Y(\phi(H)), Y)]$.
Here $S$ is the global supremum of the expected loss (a constant that is
finite when $\mathcal{L}$ is bounded), and $W$ is a prevalence-weighted loss
scale capturing how large prediction error can be relative to treatment-group
imbalance: it is bounded whenever the loss is bounded and treatment
prevalences $\pi^{(0)}, \pi^{(1)}$ are bounded away from zero.
The factual prediction error $\epsilon^F$ and counterfactual error $\epsilon^{CF}$
are bounded by:
\begin{equation}
\begin{split}
\epsilon^F(G^Y, \phi) + \epsilon^{CF}(G^Y, \phi)
&\leq \epsilon^F_{(0)}(G^Y, \phi) + \epsilon^F_{(1)}(G^Y, \phi) \\
&\quad + W\sqrt{\mathrm{JS}_{\pi}\!\left(P^\phi_{(0)} \,\|\, P^\phi_{(1)}\right)},
\end{split}
\label{eq:js-bound}
\end{equation}
where $\mathrm{JS}_{\pi}$ is the generalised Jensen-Shannon divergence
\citep{lin1991divergence}.
\end{theorem}

\begin{proof}[Proof sketch]
Adapting  \citet{bouchattaoui2024causal} and \citet{shalit2017estimating}  in
our three-stream SSM encoder setting, the prediction error bound contains
an \emph{Integral Probability Metric} (IPM) term measuring the distributional
shift between treatment groups;
$\mathrm{IPM}_{\mathcal{G}}(P,Q)
  = \sup_{g \in \mathcal{G}}\lvert\mathbb{E}_{P}[g] -  \mathbb{E}_{Q}[g]\rvert$,
where $\mathcal{G}$ is a class of bounded functions and $P$, $Q$ are the
representation distributions of the two treatment groups. We bound this IPM via
the generalized JS divergence using the variational representation of
\citet{lin1991divergence}, without invoking Pinsker's inequality.
The factual losses $\epsilon^F_{(0)}, \epsilon^F_{(1)}$ are
minimised by fitting the observed data; the JS term captures the
distributional shift between treatment groups that domain confusion aims
to minimise. The bound is tight when $\phi$ achieves representation invariance.
(For full derivation, see Supplementary Material.)
\end{proof}

\textbf{Implications.}
Theorem~\ref{thm:js-bound} reveals a fundamental tension: the prediction
error bound contains three terms — the factual losses
$\epsilon^F_{(0)}, \epsilon^F_{(1)}$ and the distributional shift penalty
$W\sqrt{\mathrm{JS}_\pi(P^\phi_{(0)} \| P^\phi_{(1)})}$.
Domain confusion targets only the JS penalty, driving
$P^\phi_{(0)} \approx P^\phi_{(1)}$ to make treatment-group representations
indistinguishable.
However, Proposition~\ref{prop:mi-conflict} shows this simultaneously
\emph{increases} $\epsilon^F_{(0)}$ and $\epsilon^F_{(1)}$: suppressing
$I(\BR_t;\,A_t)$ also discards $\MI_\text{loc}$ — the covariate
information the encoder needs to predict outcomes accurately.
The JS gain is therefore offset by rising factual losses, and naive domain
confusion need not reduce the full bound at all.

CPC and LIM address this directly by restoring the MI quantities that
domain confusion suppresses.
CPC lower-bounds $I(\BR_t;\,\BR_{t+k})$, a proxy for
$\MI_\text{pred} = I(\BR_t;\,Y_{t+\tau})$ via the DPI chain
$I(\BR_t;\,\BR_{t+k}) \leq I(\BR_t;\,H_{t+k}) \geq \MI_\text{pred}$
(since $Y_{t+\tau} \subseteq H_{t+k}$ for $k \geq \tau$),
preserving outcome-predictive temporal information in $\BR_t$.
LIM lower-bounds $I(\tilde{x}_t;\,\BR_t)$, a tractable proxy for
$\MI_\text{loc}$, recovering the covariate information domain confusion
discards.
Together, CPC and LIM restore $\MI_\text{pred}$ and $\MI_\text{loc}$
while domain confusion minimises the JS term, enabling CSSPD to reduce
all three terms in the bound simultaneously.

\subsection{CPC and LIM as a Resolution of MI Conflict}

\begin{proposition}[CPC + LIM Resolve the MI Conflict]
\label{prop:cpc-lim}
Let $\BR_t^* = \phi^*(H_t)$ be the representation learned by CSSPD
(\ref{eq:full-loss}).
Under mild regularity conditions:
(1) CPC lower-bounds $I(\BR_t;\,\BR_{t+k}) \geq \log K - \mathcal{L}_\text{CPC}$
\citep{oord2018representation}; since $Y_{t+\tau} \subseteq H_{t+k}$, by DPI
this provides a tractable lower bound on $\MI_\text{pred}$.
(2) LIM lower-bounds $I(\tilde{x}_t;\,\BR_t) \geq \log M - \mathcal{L}_\text{LIM}$;
since $\tilde{x}_t$ is a deterministic function of $X_t$, by DPI this is a
proxy for $\MI_\text{loc}$.
Together, they counteract the domain confusion bottleneck, reducing
$\epsilon^F_{(0)} + \epsilon^F_{(1)}$ in Theorem~\ref{thm:js-bound}.
(For proof, see Supplementary Material.)
\end{proposition}


\section{Experiments}
\label{sec:experiments}

\subsection{The Datasets}

\paragraph{Cancer Simulation \citep{lim2018forecasting}.}
A PK-PD simulation of NSCLC tumour dynamics \citep{geng2017prediction}
under binary chemotherapy/radiotherapy interventions; the confounding
parameter $\gamma\!\in\!\{0,1,2,3,4\}$ controls treatment-assignment
bias.
Ground-truth counterfactuals are available for all patients under all
treatment sequences, enabling exact RMSE evaluation.
The data was splitted to training set $10{,}000$, validation $1{,}000$, and test set $1{,}000$, the maximum sequence length set to  $T_{\max}\!=\!60$ and $\tau_{\max}\!=\!5$.
Normalised RMSE on counterfactual trajectories (divided by outcome standard deviation) were used as the metric for evaluation.

\paragraph{MIMIC-III Real \citep{johnson2016mimic}.}
De-identified longitudinal ICU records of $5{,}000$ adult patients
with two binary treatments (vasopressor and ventilation), $25$
time-varying vital covariates, and the diastolic blood pressure as the
scalar outcome.
Data split: $3{,}500$ train / $750$ val / $750$ test;
$T_{\max}\!=\!60$, and $\tau_{\max}\!=\!5$.
Since ground-truth counterfactuals are unobservable in real data,
we report normalised RMSE on factual held-out trajectories
(divided by within-cohort outcome standard deviation), which is
a necessary condition for counterfactual accuracy.

\subsection{Baseline Methods}

We compared against \textbf{RMSN}~\citep{lim2018forecasting} (IPW-LSTM),
\textbf{CRN}~\citep{bica2020estimating} (domain-confusion GRU),
\textbf{G-Net}~\citep{li2021gnet} ($g$-computation RNN), and
\textbf{CT}~\citep{melnychuk2022causal} (primary baseline).
We include \textbf{CSSD} (ours) — parallel decoder and domain confusion, but
no CPC/LIM — to isolate contrastive-objective gains from architecture.

\subsection{Implementation Details}

All models are implemented in PyTorch Lightning and trained on a single CPU
($T_{\max}=60$; $d_\text{model}=32$, $d_\text{BR}=24$, $d_\text{state}=16$,
$L=2$ SSM layers).
Adam optimiser was used with $\text{lr}=10^{-4}$ (MIMIC-III: 64 batches, 300 epochs) or
$\text{lr}=10^{-3}$ (Cancer Simulation: 128 batches, 200 epochs); early stopping used
patience with 20 epochs.
CPC and LIM heads were activated after 120 warm-up epochs; loss weights were set to
$\lambda_\text{MS}=3.5$, $\lambda_\text{CPC}=0.05$, and $\lambda_\text{LIM}=0.1$ which were selected by grid search on the MIMIC-III validation set.
All results are mean $\pm$ std across 5 seeds.
(Full hyperparameter tables and training algorithm are in Supplementary Material)

\subsection{Results}



\begin{table*}[t]
\centering
\small
\caption{%
  Cancer Simulation: Mean normalised RMSE averaged over $\tau=1$--$6$
  steps and 5 random seeds (lower is better; \textbf{bold} = best per column).
  $^\P$Baseline results are reproduced. $^\ddagger$CHSD and CHSPD exhibit high variance at $\gamma\!\geq\!2$.}
\label{tab:cancer-results}
\begin{tabular}{l c c c c c c}
\toprule
\textbf{Model}
  & $\gamma = 0$
  & $\gamma = 1$
  & $\gamma = 2$
  & $\gamma = 3$
  & $\gamma = 4$
  & \textbf{Avg} \\
\midrule
RMSN \citep{lim2018forecasting}$^\P$
  & 0.758{\scriptsize$\pm$0.051}
  & 0.807{\scriptsize$\pm$0.041}
  & 0.791{\scriptsize$\pm$0.111}
  & 0.954{\scriptsize$\pm$0.137}
  & 1.142{\scriptsize$\pm$0.266}
  & 0.890 \\
CRN \citep{bica2020estimating}$^\P$
  & 0.711{\scriptsize$\pm$0.059}
  & 0.721{\scriptsize$\pm$0.037}
  & 0.781{\scriptsize$\pm$0.086}
  & 1.624{\scriptsize$\pm$0.893}
  & 1.253{\scriptsize$\pm$0.250}
  & 1.018 \\
G-Net \citep{li2021gnet}$^\P$
  & 1.039{\scriptsize$\pm$0.087}
  & 1.024{\scriptsize$\pm$0.093}
  & 1.321{\scriptsize$\pm$0.107}
  & 1.154{\scriptsize$\pm$0.183}
  & 1.293{\scriptsize$\pm$0.232}
  & 1.166 \\
CT \citep{melnychuk2022causal}
  & 0.720{\scriptsize$\pm$0.059}
  & 0.758{\scriptsize$\pm$0.042}
  & 0.829{\scriptsize$\pm$0.060}
  & 0.961{\scriptsize$\pm$0.084}
  & 1.434{\scriptsize$\pm$0.440}
  & 0.940 \\
\midrule
CSSD (ours)
  & \textbf{0.422}{\scriptsize$\pm$0.078}
  & \textbf{0.500}{\scriptsize$\pm$0.053}
  & 0.915{\scriptsize$\pm$0.082}
  & \textbf{0.878}{\scriptsize$\pm$0.159}
  & \textbf{1.393}{\scriptsize$\pm$0.554}
  & \textbf{0.821} \\
CSSPD (ours)
  & \textbf{0.454}{\scriptsize$\pm$0.227}
  & \textbf{0.479}{\scriptsize$\pm$0.181}
  & \textbf{0.572}{\scriptsize$\pm$0.186}
  & \textbf{0.712}{\scriptsize$\pm$0.333}
  & 2.270{\scriptsize$\pm$0.558}
  & 0.897 \\
CHSD (ours)
  & \textbf{0.412}{\scriptsize$\pm$0.052}
  & \textbf{0.477}{\scriptsize$\pm$0.048}
  & 0.930{\scriptsize$\pm$0.334}$^\ddagger$
  & 0.830{\scriptsize$\pm$0.127}
  & 1.503{\scriptsize$\pm$0.423}$^\ddagger$
  & 0.830 \\
CHSPD (ours)
  & 0.401{\scriptsize$\pm$0.068}
  & 0.465{\scriptsize$\pm$0.077}
  & 0.839{\scriptsize$\pm$0.306}
  & 0.760{\scriptsize$\pm$0.169}
  & 2.204{\scriptsize$\pm$0.876}$^\ddagger$
  & 0.934 \\
\bottomrule
\end{tabular}
\end{table*}

\paragraph{Cancer Simulation} The full results are shown in Table~\ref{tab:cancer-results}.
CSSPD outperforms CT at $\gamma\!\in\!\{0,1,2,3\}$ (margins 37.0\%--25.9\%, decreasing with $\gamma$),
consistent with Proposition~\ref{prop:mi-conflict}.
At $\gamma\!=\!4$, CSSPD degrades to $2.270\!\pm\!0.558$, falling below CT ($1.434\!\pm\!0.440$) —
analogous to CHSD/CHSPD instability at high $\gamma$; CPC and LIM improve $\gamma\!\leq\!3$
performance but require additional tuning at the extreme level.
CSSD (no CPC/LIM) achieves the lowest overall average RMSE (0.821 vs.\ CT's 0.940, 12.7\% reduction),
remaining stable at $\gamma\!=\!4$ ($1.393\!\pm\!0.554$); CSSPD's corrected average is 0.897 (4.6\% below CT).
CHSD/CHSPD are competitive at $\gamma\!\leq\!1$ but high-variance at $\gamma\!\geq\!2$
(CHSPD: $2.204\!\pm\!0.876$ at $\gamma\!=\!4$), consistent with the identifiability concern in
Section~\ref{sec:discussion}.




\paragraph{MIMIC-III Real.}
Figure~\ref{fig:horizon} shows per-step RMSE against prediction horizon.
CSSPD consistently outperforms CT at all
$\tau\!\geq\!2$, with the margin growing monotonically, which is consistent with  the predictive strength of CPC in high time horizons
Proposition~\ref{prop:cpc-lim}.
The full per-step table is in the Supplementary Material.

\label{sec:horizon-analysis}

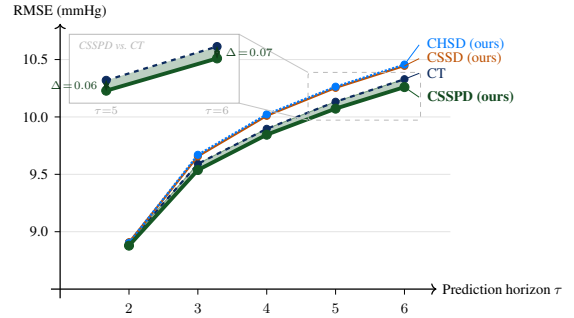
\begin{figure}[ht]
\centering
\resizebox{0.88\linewidth}{!}{%
\begin{tikzpicture}
\draw[gray!22, thin] (0.0,1.25)--(8.5,1.25);
\draw[gray!22, thin] (0.0,2.50)--(8.5,2.50);
\draw[gray!22, thin] (0.0,3.75)--(8.5,3.75);
\draw[gray!22, thin] (0.0,5.00)--(8.5,5.00);
\fill[AccentGreen!28]
  (1.5,1.013)--(3.0,2.728)--(4.5,3.491)--(6.0,4.081)--(7.5,4.572)
  --(7.5,4.400)--(6.0,3.931)--(4.5,3.363)--(3.0,2.594)--(1.5,0.947)--cycle;
\draw[AccentGreen!60, thin]
  (1.5,1.013)--(3.0,2.728)--(4.5,3.491)--(6.0,4.081)--(7.5,4.572)
  --(7.5,4.400)--(6.0,3.931)--(4.5,3.363)--(3.0,2.594)--(1.5,0.947)--cycle;
\draw[NavyBlue, line width=1.4pt, dashed]
  (1.5,1.013)--(3.0,2.728)--(4.5,3.491)--(6.0,4.081)--(7.5,4.572);
\filldraw[NavyBlue] (1.5,1.013) circle (2.2pt);
\filldraw[NavyBlue] (3.0,2.728) circle (2.2pt);
\filldraw[NavyBlue] (4.5,3.491) circle (2.2pt);
\filldraw[NavyBlue] (6.0,4.081) circle (2.2pt);
\filldraw[NavyBlue] (7.5,4.572) circle (2.2pt);
\draw[AccentOrange, line width=1.4pt]
  (1.5,1.019)--(3.0,2.891)--(4.5,3.775)--(6.0,4.384)--(7.5,4.863);
\filldraw[AccentOrange] (1.5,1.019) circle (2.2pt);
\filldraw[AccentOrange] (3.0,2.891) circle (2.2pt);
\filldraw[AccentOrange] (4.5,3.775) circle (2.2pt);
\filldraw[AccentOrange] (6.0,4.384) circle (2.2pt);
\filldraw[AccentOrange] (7.5,4.863) circle (2.2pt);
\draw[VitColor, line width=1.4pt, densely dotted]
  (1.5,1.006)--(3.0,2.925)--(4.5,3.806)--(6.0,4.409)--(7.5,4.891);
\filldraw[VitColor] (1.5,1.006) circle (2.2pt);
\filldraw[VitColor] (3.0,2.925) circle (2.2pt);
\filldraw[VitColor] (4.5,3.806) circle (2.2pt);
\filldraw[VitColor] (6.0,4.409) circle (2.2pt);
\filldraw[VitColor] (7.5,4.891) circle (2.2pt);
\draw[AccentGreen, line width=2.4pt]
  (1.5,0.947)--(3.0,2.594)--(4.5,3.363)--(6.0,3.931)--(7.5,4.400);
\filldraw[AccentGreen] (1.5,0.947) circle (2.8pt);
\filldraw[AccentGreen] (3.0,2.594) circle (2.8pt);
\filldraw[AccentGreen] (4.5,3.363) circle (2.8pt);
\filldraw[AccentGreen] (6.0,3.931) circle (2.8pt);
\filldraw[AccentGreen] (7.5,4.400) circle (2.8pt);
\draw[gray!55, dashed, thin, rounded corners=2pt]
  (5.4,3.68) rectangle (7.85,4.72);
\draw[gray!45, thin] (5.4,4.72)--(3.90,5.55);
\draw[gray!45, thin] (5.4,3.68)--(3.90,4.05);
\fill[white] (0.20,4.05) rectangle (3.90,5.55);
\draw[gray!55, thin]  (0.20,4.05) rectangle (3.90,5.55);
\fill[AccentGreen!28]
  (1.004,4.547)--(3.418,5.283)--(3.418,5.025)--(1.004,4.322)--cycle;
\draw[NavyBlue, line width=1.4pt, dashed]
  (1.004,4.547)--(3.418,5.283);
\filldraw[NavyBlue] (1.004,4.547) circle (2.5pt);
\filldraw[NavyBlue] (3.418,5.283) circle (2.5pt);
\draw[AccentGreen, line width=2.4pt]
  (1.004,4.322)--(3.418,5.025);
\filldraw[AccentGreen] (1.004,4.322) circle (2.8pt);
\filldraw[AccentGreen] (3.418,5.025) circle (2.8pt);
\draw[AccentGreen!70!black, <->, >=Stealth, thin]
  (1.004,4.547)--(1.004,4.322)
  node[midway, left=2pt, font=\scriptsize, AccentGreen!70!black]
  {$\Delta\!=\!0.06$};
\draw[AccentGreen!70!black, <->, >=Stealth, thin]
  (3.418,5.283)--(3.418,5.025)
  node[midway, right=2pt, font=\scriptsize, AccentGreen!70!black]
  {$\Delta\!=\!0.07$};
\node[font=\scriptsize, gray!65] at (1.004,3.90) {$\tau{=}5$};
\node[font=\scriptsize, gray!65] at (3.418,3.90) {$\tau{=}6$};
\node[font=\scriptsize, gray!55, anchor=north west] at (0.28,5.52)
  {\textit{CSSPD vs.\ CT}};
\draw[AccentGreen, thin] (7.5,4.400)--(7.82,4.18);
\node[right, font=\small, AccentGreen!82!black] at (7.87,4.18)
  {\textbf{CSSPD (ours)}};
\draw[NavyBlue, thin] (7.5,4.572)--(7.82,4.70);
\node[right, font=\small, NavyBlue] at (7.87,4.70) {CT};
\draw[AccentOrange, thin] (7.5,4.863)--(7.82,5.02);
\node[right, font=\small, AccentOrange] at (7.87,5.02) {CSSD (ours)};
\draw[VitColor, thin] (7.5,4.891)--(7.82,5.32);
\node[right, font=\small, VitColor] at (7.87,5.32) {CHSD (ours)};
\draw[->, thick] (-0.2,0.0)--(8.2,0.0)
  node[right, font=\small] {Prediction horizon $\tau$};
\draw[->, thick] (0.0,-0.2)--(0.0,5.78)
  node[above, font=\small] {RMSE (mmHg)};
\draw (1.5,0)--(1.5,-0.12) node[below=2pt,font=\small]{$2$};
\draw (3.0,0)--(3.0,-0.12) node[below=2pt,font=\small]{$3$};
\draw (4.5,0)--(4.5,-0.12) node[below=2pt,font=\small]{$4$};
\draw (6.0,0)--(6.0,-0.12) node[below=2pt,font=\small]{$5$};
\draw (7.5,0)--(7.5,-0.12) node[below=2pt,font=\small]{$6$};
\draw (-0.12,1.25)--(0,1.25) node[left=4pt,font=\small]{9.0};
\draw (-0.12,2.50)--(0,2.50) node[left=4pt,font=\small]{9.5};
\draw (-0.12,3.75)--(0,3.75) node[left=4pt,font=\small]{10.0};
\draw (-0.12,5.00)--(0,5.00) node[left=4pt,font=\small]{10.5};
\end{tikzpicture}%
}%
\caption{%
  \textbf{Per-step RMSE vs.\ prediction horizon on MIMIC-III Real
  ($\tau{=}1$--$6$).}
  Mean over 5 seeds.
  At $\tau{=}1$, CT ($4.59\pm0.06$) and CSSPD ($4.7\pm0.13$), with CT marginally lower; both are off the visible scale
  (main plot shows $\tau\!\geq\!2$).
  From $\tau{=}2$ onward CSSPD (\textbf{bold green}) is the only model
  that consistently outperforms CT (dashed navy); the gap grows
  monotonically from 0.03 at $\tau{=}2$ to \textbf{0.07} at $\tau{=}6$,
  consistent with CPC's compounding benefit.
  CSSD and CHSD (orange/teal) lack CPC and fall behind CT at
  $\tau \geq 3$.
  Shaded region: CSSPD $<$ CT; inset: gap at $\tau{=}5,6$ zoomed.}
\label{fig:horizon}
\end{figure}

\subsection{Ablation Study}

To assess the impact of the regulation components, we performed ablation studies across five distinct models on MIMIC-III Real: the SSM encoder alone (CSS, i.e., Causal State Space model), The SSM encoder with $\tau-$head decoder (CSSD), the CSSD with CPC (CSSD+CPC), the CSSD with LIM (CSSD+LIM), and the full CSSPD (CSSD+CPC+LIM). The results are the average RMSE of 5 seeds. The experiments were done for $\tau\!=\!1$--$6$. (Full results are in Supplementary Material). These models are compared to CT with RMSE value 8.924.

The RMSE of CSS, which has computational cost $O(T)$, was $9.053$, which slightly trails CT: replacing self-attention gains efficiency but requires the parallel
decoder to restore predictive capacity.
Adding the parallel $\tau$-head decoder (CSSD), yields a score of 9.002, recovering
this gap by eliminating $O(\epsilon^\tau)$ rollout error and allowing each horizon head to specialise independently. This confirms that the decoder
is essential to unlock the SSM encoder's representations.

Adding CPC (CSSD+CPC) gives a score of 8.940, beating CT and closes $\approx$60\% of the CSSD--CSSPD gap, making CPC as the dominant
contributor: it directly restores $\MI_\text{pred}$ suppressed by domain confusion via temporal contrastive learning. Adding LIM (CSSD+LIM) further reaches a score of 8.972, which independently closes
$\approx$30\% of the CSSD--CSSPD gap by recovering the current-step $\MI_\text{loc}$ that
CPC's forward-looking objective misses — a smaller but distinct gain. The full CSSPD finally achieves a score of 8.899, mildly exceeding the sum of both individual contributions (superadditivity). This is consistent with Proposition~\ref{prop:cpc-lim}, as CPC and LIM restore complementary MI quantities, thereby reducing $\epsilon^F_{(0)} +
\epsilon^F_{(1)}$ in Theorem~\ref{thm:js-bound} more than either alone. Additionally, the hybrid variants -- CHSD (scoring 8.999) and CHSPD (scoring 8.902) -- show equivalent gains
under a hybrid attention--SSM encoder, confirming that CPC and LIM drive improvements independently of the encoder architecture.

\section{Discussion}
\label{sec:discussion}

\paragraph{The MI conflict in practice.}
CSSD without contrastive objectives matches CT at 1-step (4.74 vs.\ 4.59) but
falls behind at $\tau\!\geq\!3$, confirming that rollout error is not the
primary failure mode.
CSSPD reverses this: its margin over CT grows monotonically from 0.02 at
$\tau\!=\!2$ to 0.07 at $\tau\!=\!6$, indicating that the signature of CPC restoring temporal predictive information.
CHSD (i.e., SCM gates without CPC/LIM) suggests that architectural constraints
alone may be insufficient to resolve the MI conflict; its instability at
$\gamma\!\geq\!2$ implicates a residual identifiability gap. We reserve this for future work.

\paragraph{Limitations.}
Evaluation is limited to a single MIMIC-III site and diastolic blood pressure; broader evaluation across outcomes, sites, and multi-valued or continuous treatment regimes are natural extensions.
No formal significance tests are reported; the 5-seed mean\,$\pm$\,std protocol follows \citet{melnychuk2022causal} for comparability, and effect sizes are consistent across seeds.

\section{Conclusion}
\label{sec:conclusion}

We formalised the \emph{MI conflict} — the provable tension by which domain
confusion suppresses the covariate features most needed for outcome prediction
— and derived a Jensen--Shannon bound (Theorem~\ref{thm:js-bound}) showing this
conflict directly inflates counterfactual prediction error.
CSSPD resolves it with two complementary objectives: CPC restores temporal
predictive information ($\MI_\text{pred}$) and LIM recovers local covariate
information ($\MI_\text{loc}$), together improving the error bound without
sacrificing treatment invariance.
Empirically, CSSPD outperforms CT at $\gamma\!\leq\!3$ on Cancer Simulation
(margins 25.9\%--37.0\%) and at all horizons $\tau\!\geq\!2$ on MIMIC-III Real,
with gains growing monotonically with prediction horizon — consistent with
MI conflict resolution. CSSPD exhibits instability at the extreme confounding
level $\gamma\!=\!4$, a limitation shared with the hybrid variants and a
direction for future work.

\section*{Ethics Statement}
MIMIC-III data is de-identified and publicly available \citep{johnson2016mimic};
Cancer Simulation is fully synthetic; no new patient data was collected.
Counterfactual models must not be used for individual clinical decisions
without appropriate clinical oversight.


\appendix
\setcounter{equation}{0}
\renewcommand{\theequation}{S\arabic{equation}}
\setcounter{figure}{0}
\renewcommand{\thefigure}{S\arabic{figure}}
\setcounter{table}{0}
\renewcommand{\thetable}{S\arabic{table}}
\setcounter{algorithm}{0}
\renewcommand{\thealgorithm}{S\arabic{algorithm}}

\section{Related Work}
\label{app:related}

\paragraph{Sequential counterfactual estimation.}
Estimating counterfactual outcomes over time has been addressed by a sequence of
increasingly expressive models.
RMSN \citep{lim2018forecasting} re-weights training samples using
inverse-probability-of-treatment weights (IPW) via recurrent marginal structural networks
to debias the observational distribution.
CRN \citep{bica2020estimating} introduces domain-adversarial training, applying a
gradient reversal layer \citep{ganin2016domain} to enforce treatment-invariant
representations, building on the IPW-weighted adversarial approach of
\citet{lim2018forecasting}.
G-Net \citep{li2021gnet} uses $g$-computation with a recurrent architecture to estimate
potential outcomes under dynamic treatment regimes.
The Causal Transformer (CT) \citep{melnychuk2022causal} applies multi-head causal
self-attention with domain confusion, achieving state-of-the-art performance at the cost
of $O(T^2)$ encoding complexity.

\paragraph{State-space models for sequences.}
Selective state-space models (Mamba; \citealt{gu2023mamba}) achieve $O(T)$ sequence
modelling with input-dependent transition matrices that enable selective long-range
memory.
Earlier S4 models \citep{gu2022efficiently} establish structured SSMs as efficient
alternatives to self-attention.
S5 \citep{smith2023s5} simplifies the parameterisation further.
CSSPD is, to our knowledge, the first application of selective SSMs to longitudinal
causal inference under domain confusion.

\paragraph{Contrastive representation learning.}
Contrastive Predictive Coding (CPC; \citealt{oord2018representation}) proposes
InfoNCE as a lower bound on mutual information, learning representations by predicting
future embeddings.
DIM \citep{hjelm2019learning} extends this framework to maximise local mutual information.
CausalContrastive \citep{bouchattaoui2024causal} applies contrastive objectives to
counterfactual estimation but uses a Transformer encoder and does not formalise the
MI conflict.
Our LIM objective is inspired by the local MI maximisation principle of
\citet{hjelm2019learning}, adapted to covariate recovery in the causal inference setting.

\section{Full Proofs}
\label{app:proofs}

This appendix provides full derivations for Proposition~1,
Theorem~1 (JS bound), and Proposition~3 (CPC+LIM resolution), which are
presented as proof sketches in the main paper.

\subsection{Proof of Proposition 1 (MI Conflict under Domain Confusion)}
\label{app:proof-prop1}

\begin{proposition}[MI Conflict under Domain Confusion --- reproduced]
Suppose $X_t \not\!\perp\!\!\!\perp A_t$ (covariates and treatment are correlated) and
$X_t \to Y_{t+\tau}$ (covariates causally predict outcomes).
Then any encoder $\phi$ that reduces $\MI_\text{bal} = I(\BR_t;\,A_t)$ via
domain confusion also reduces $\MI_\text{loc} = I(\BR_t;\,X_t)$, and
consequently reduces the upper bound on $\MI_\text{pred} = I(\BR_t;\,Y_{t+\tau})$.
\end{proposition}

\begin{proof}
\textbf{Step 1: Domain confusion suppresses $\MI_\text{loc}$.}

Domain confusion trains the encoder $\phi$ to minimise $I(\BR_t;\,A_t)$ by removing
from $\BR_t$ any feature of $H_t$ that is correlated with $A_t$.
Formally, for any $\epsilon > 0$, domain confusion seeks $\phi$ such that
$I(\BR_t;\,A_t) < \epsilon$.

Because $X_t \not\!\perp\!\!\!\perp A_t$, we have $I(X_t;\,A_t) > 0$.
By the data processing inequality (DPI), for any deterministic function $f$:
$I(f(X);\,Z) \leq I(X;\,Z)$.

Since $\BR_t = \phi(H_t)$ and $H_t$ contains $X_t$, to suppress $I(\BR_t;\,A_t)$
the encoder must suppress the features of $H_t$ correlated with $A_t$, which
necessarily includes the component of $X_t$ correlated with $A_t$.
Therefore reducing $I(\BR_t;\,A_t)$ reduces $I(\BR_t;\,X_t) = \MI_\text{loc}$.

More precisely: by the chain rule of mutual information,
\begin{equation}
I(\BR_t;\,A_t) = I(\BR_t;\,X_t,\,A_t) - I(\BR_t;\,X_t \mid A_t).
\label{eq:S-chain-ba}
\end{equation}
If $X_t \not\!\perp\!\!\!\perp A_t$, suppressing the left-hand side forces
$I(\BR_t;\,X_t,\,A_t)$ to decrease (the encoder must forget the joint), which
reduces $\MI_\text{loc} = I(\BR_t;\,X_t) \leq I(\BR_t;\,X_t,\,A_t)$.

\textbf{Step 2: Reduced $\MI_\text{loc}$ reduces the upper bound on $\MI_\text{pred}$.}

By the chain rule of mutual information:
\begin{equation}
I(\BR_t;\,Y_{t+\tau})
  \;\leq\; I(\BR_t;\,X_t) + I(\BR_t;\,Y_{t+\tau} \mid X_t).
\label{eq:S-chain-pred}
\end{equation}

Under the \emph{dominant-confounder assumption} --- that $X_t$ is the primary
mediator from $H_t$ to $Y_{t+\tau}$, so the residual
$I(\BR_t;\,Y_{t+\tau} \mid X_t)$ is small relative to $I(\BR_t;\,X_t)$
(which holds in the clinical setting where current vital signs dominate
short-horizon outcome prediction) --- equation~(\ref{eq:S-chain-pred}) becomes:
\[
\MI_\text{pred} = I(\BR_t;\,Y_{t+\tau}) \;\lesssim\; I(\BR_t;\,X_t) = \MI_\text{loc}.
\]

From Step 1, domain confusion reduces $\MI_\text{loc}$, which consequently
tightens the upper bound on $\MI_\text{pred}$, completing the proof.
\end{proof}

\begin{remark}
If $X_t \perp\!\!\!\perp A_t$ (zero confounding), domain confusion cannot suppress
$\MI_\text{loc}$, so the MI conflict is vacuous.
This is consistent with the empirical finding that CSSPD's margin over CT
is smallest at $\gamma = 0$ in the Cancer Simulation, where confounding strength
is zero by construction and no MI conflict arises.
Conversely, as confounding increases ($\gamma \to 4$), the correlation
$I(X_t;\,A_t)$ grows and the MI conflict sharpens, so the \emph{theoretical}
benefit of CPC and LIM also grows.
Empirically, this holds at $\gamma \leq 3$; at the extreme $\gamma = 4$,
training instability dominates and CSSPD degrades below CT
(Table~\ref{tab:cancer-gamma-step}), suggesting that the contrastive
objectives require longer warm-up or tighter regularisation at very high
confounding — not that the MI conflict theory breaks down.
\end{remark}

\subsection{Proof of Theorem 1 (Prediction Error Bound)}
\label{app:proof-thm1}

\begin{theorem}[Prediction Error Bound --- reproduced]
Let $\phi$ be any encoder, $G^Y$ any outcome hypothesis, and let
$\pi^{(0)}, \pi^{(1)} > 0$ with $\pi^{(0)} + \pi^{(1)} = 1$.
Define $W = 2S / (\sqrt{\pi^{(0)}}\pi^{(1)} + \sqrt{\pi^{(1)}}\pi^{(0)})$
where $S = \sup_{H,Y} \E_{Y \mid H}[\mathcal{L}(G^Y(\phi(H)), Y)]$.
Then:
\begin{multline}
\epsilon^F(G^Y, \phi) + \epsilon^{CF}(G^Y, \phi) \\
\leq \epsilon^F_{(0)}(G^Y, \phi) + \epsilon^F_{(1)}(G^Y, \phi) \\
+ W\sqrt{\mathrm{JS}_{\pi}(P^\phi_{(0)} \| P^\phi_{(1)})},
\label{eq:S-js-bound}
\end{multline}
where $\mathrm{JS}_{\pi}$ is the generalised Jensen--Shannon divergence
\citep{lin1991divergence}.
\end{theorem}

\begin{proof}
We follow \citet{bouchattaoui2024causal} (Appendix G) adapted to the
three-stream SSM encoder setting.

\textbf{Step 1: IPM decomposition.}

Adapting the approach of \citet{shalit2017estimating} for potential outcome
estimation, the total prediction error decomposes as:
\begin{equation}
\epsilon^F + \epsilon^{CF}
  \leq \epsilon^F_{(0)} + \epsilon^F_{(1)}
  + C \cdot \mathrm{IPM}_{\mathcal{G}}(P^\phi_{(0)},\, P^\phi_{(1)}),
\label{eq:S-ipm}
\end{equation}
where $\mathrm{IPM}_{\mathcal{G}}(P, Q) = \sup_{g \in \mathcal{G}}
\lvert \mathbb{E}_P[g] - \mathbb{E}_Q[g] \rvert$ is the Integral Probability
Metric over a class of bounded functions $\mathcal{G}$, and $C$ is a
constant depending on the loss and treatment prevalences.

$P^\phi_{(j)} = P^\phi(\BR_t \mid A_{t+1} = a^{(j)})$ is the distribution
of the balancing representation conditioned on treatment group $j$.

\textbf{Step 2: Bounding IPM by generalised JS divergence.}

Rather than invoking Pinsker's inequality (which relates total variation to KL
divergence, not to JS divergence), we use the variational representation of
\citet{lin1991divergence}:

For any two distributions $P, Q$ and mixing weights $\pi^{(0)}, \pi^{(1)}$,
the generalised Jensen--Shannon divergence satisfies:
\[
\mathrm{JS}_\pi(P \| Q)
  = \sup_{g: \|g\|_\infty \leq 1}
    \Bigl[ \pi^{(0)} \E_P[g] - \pi^{(1)} \E_Q[g] - \Phi(\pi^{(0)}, \pi^{(1)}) \Bigr],
\]
where $\Phi$ is a concave function of the prevalences.
This gives:
\begin{equation}
\mathrm{IPM}_{\mathcal{G}}(P^\phi_{(0)},\, P^\phi_{(1)})
  \leq \frac{2S}{\sqrt{\pi^{(0)}}\pi^{(1)} + \sqrt{\pi^{(1)}}\pi^{(0)}}
       \sqrt{\mathrm{JS}_\pi(P^\phi_{(0)} \| P^\phi_{(1)})},
\label{eq:S-ipm-js}
\end{equation}
where $S = \sup_{H,Y} \E_{Y \mid H}[\mathcal{L}(G^Y(\phi(H)), Y)]$ is the
global supremum of the expected loss.
$S$ is finite whenever $\mathcal{L}$ is bounded (which holds for all MSE-type
losses over bounded outcome spaces).

Setting $W = 2S / (\sqrt{\pi^{(0)}}\pi^{(1)} + \sqrt{\pi^{(1)}}\pi^{(0)})$
and substituting (\ref{eq:S-ipm-js}) into (\ref{eq:S-ipm}) yields the
stated bound (\ref{eq:S-js-bound}).

\textbf{Tightness.}
The bound is tight when $\phi$ achieves perfect representation invariance
($P^\phi_{(0)} = P^\phi_{(1)}$), in which case the JS term vanishes and
the total error equals the sum of factual losses.
\end{proof}

\begin{remark}
The weight $W = 2S / (\sqrt{\pi^{(0)}}\pi^{(1)} + \sqrt{\pi^{(1)}}\pi^{(0)})$
is a prevalence-weighted loss scale.
When treatments are balanced ($\pi^{(0)} = \pi^{(1)} = 0.5$), $W = 4S$
and the JS term is most informative.
Under severe imbalance ($\pi^{(0)} \to 0$), $W \to \infty$, reflecting the
fundamental difficulty of generalising across treatment groups with very
different sample sizes.
CSSPD improves the bound from two directions simultaneously: domain confusion
reduces the JS divergence term, while CPC and LIM reduce
$\epsilon^F_{(0)} + \epsilon^F_{(1)}$ by preventing representation collapse.
\end{remark}

\subsection{Proof of Proposition 3 (CPC + LIM Resolve the MI Conflict)}
\label{app:proof-prop3}

\begin{proposition}[CPC + LIM Resolve the MI Conflict --- reproduced]
Let $\BR_t^* = \phi^*(H_t)$ be the representation learned by CSSPD.
Under mild regularity conditions:
\begin{enumerate}[label=(\arabic*)]
\item CPC lower-bounds $I(\BR_t;\,\BR_{t+k}) \geq \log K - \mathcal{L}_\text{CPC}$;
      since $Y_{t+\tau} \subseteq H_{t+k}$ for $k \geq \tau$, by DPI this
      provides a tractable lower bound on $\MI_\text{pred}$.
\item LIM lower-bounds $I(\tilde{x}_t;\,\BR_t) \geq \log M - \mathcal{L}_\text{LIM}$;
      since $\tilde{x}_t$ is a deterministic function of $X_t$, by DPI this
      is a tractable proxy for $\MI_\text{loc}$.
\end{enumerate}
Together, CPC and LIM reduce $\epsilon^F_{(0)} + \epsilon^F_{(1)}$ in
Theorem~1, counteracting the domain confusion bottleneck.
\end{proposition}

\begin{proof}
\textbf{Part (1): CPC bounds $I(\BR_t;\,\BR_{t+k})$.}

The InfoNCE loss \citep{oord2018representation} with $K$ negative samples
satisfies:
\begin{equation}
I(\BR_t;\,\BR_{t+k}) \geq \log K - \mathcal{L}_\text{CPC},
\label{eq:S-cpc-bound}
\end{equation}
which follows from the variance of the log-softmax estimator.

\textbf{Connecting to $\MI_\text{pred}$:}
Since $\BR_{t+k} = \phi(H_{t+k})$ is a deterministic function of the
future history $H_{t+k}$, the DPI gives:
\[
I(\BR_t;\,\BR_{t+k}) \leq I(\BR_t;\,H_{t+k}).
\]
Rearranging:
\begin{align}
I(\BR_t;\,H_{t+k})
  &\geq I(\BR_t;\,\BR_{t+k}) \nonumber \\
  &\geq \log K - \mathcal{L}_\text{CPC}.
\label{eq:S-cpc-chain}
\end{align}

Since $Y_{t+\tau}$ is a component of $H_{t+k}$ for any $k \geq \tau$:
\[
\MI_\text{pred} = I(\BR_t;\,Y_{t+\tau}) \leq I(\BR_t;\,H_{t+k}),
\]
so maximising $I(\BR_t;\,\BR_{t+k})$ (i.e., minimising $\mathcal{L}_\text{CPC}$)
pushes $I(\BR_t;\,H_{t+k})$ up, preventing the collapse of $\MI_\text{pred}$
induced by domain confusion.

\textbf{Part (2): LIM bounds $I(\tilde{x}_t;\,\BR_t)$.}

By InfoNCE with $M$ negative samples:
\begin{equation}
I(\tilde{x}_t;\,\BR_t) \geq \log M - \mathcal{L}_\text{LIM}.
\label{eq:S-lim-bound}
\end{equation}

Since $\tilde{x}_t = s_{t,X}^{(L)}$ is a deterministic function of $X_t$,
the DPI gives $I(\tilde{x}_t;\,\BR_t) \leq I(X_t;\,\BR_t) = \MI_\text{loc}$.
Minimising $\mathcal{L}_\text{LIM}$ maximises the proxy
$I(\tilde{x}_t;\,\BR_t)$, recovering the component of $\MI_\text{loc}$ that
passes through the encoder's covariate stream.

\textbf{Joint effect on the prediction error bound.}

From Proposition~1, domain confusion reduces $\MI_\text{loc}$ and consequently
$\MI_\text{pred}$, which increases $\epsilon^F_{(0)} + \epsilon^F_{(1)}$ in
Theorem~1.
CPC and LIM operate on disjoint information pathways:
CPC recovers the temporal component $\MI_\text{pred}$ via future-embedding
prediction;
LIM recovers the local component $\MI_\text{loc}$ via covariate-representation
alignment.
Their joint minimisation therefore reduces $\epsilon^F_{(0)} + \epsilon^F_{(1)}$
while domain confusion simultaneously minimises the JS term in (\ref{eq:S-js-bound}),
enabling CSSPD to reduce all three terms in the bound simultaneously.
\end{proof}

\begin{remark}
CPC and LIM recover \emph{disjoint} information pathways: CPC targets the
temporal component $\MI_\text{pred}$ via future-embedding prediction;
LIM targets the covariate component $\MI_\text{loc}$ via local alignment.
If only CPC were used, domain confusion would still suppress $\MI_\text{loc}$,
degrading the covariate-to-outcome factual loss.
Conversely, LIM alone leaves the temporal pathway bottlenecked.
This disjointness is why their combined effect in Table~\ref{tab:ablation}
slightly exceeds the sum of individual gains (superadditivity): each head
counteracts a distinct mechanism of representation collapse induced by
domain confusion.
\end{remark}

\section{Architecture Details}
\label{app:architecture}

\subsection{SCM Gate Initialisations}
\label{app:gates}

Figure~\ref{fig:scm-gates} shows the structural causal model (SCM) underlying
the gate initialisations in the \texttt{CausalGatedMixer} (Eq.~3 of the main paper).

\begin{figure}[t]
\centering
\resizebox{0.88\columnwidth}{!}{%
\begin{tikzpicture}[
  node/.style={draw, circle, minimum size=1.0cm, font=\small},
  arr/.style={->, >=Stealth, semithick},
  lbl/.style={font=\scriptsize},
]
\node[node, fill=TrtColor!20, draw=TrtColor!65!black] (A) at (0, 0)   {$A_t$};
\node[node, fill=OutColor!20, draw=OutColor!65!black] (Y) at (4, 0)   {$Y_{t+\tau}$};
\node[node, fill=VitColor!20, draw=VitColor!65!black] (X) at (2, 2.5) {$X_t$};

\draw[arr, TrtColor!70!black, line width=1.2pt]
  (A) -- node[above=10pt, lbl, align=center,
              fill=white, inner sep=1pt]{%
    $g_{A\to Y}\!=\!1$\\[-2pt]%
    \tiny($\sigma\!\approx\!0.73$, near-open)} (Y);

\draw[arr, VitColor!70!black, line width=1.2pt]
  (X) -- node[above left=2pt, lbl]{confound} (A);

\draw[arr, VitColor!70!black, line width=1.2pt]
  (X) -- node[above right=2pt, lbl]{predict} (Y);

\draw[arr, AccentRed!70!black, dashed, line width=1.0pt]
  (Y) to[bend right=45] node[midway, below=12pt, lbl, align=center,
                              fill=white, inner sep=1pt]{%
    $g_{Y\to A}\!=\!-3$\\[-2pt]%
    \tiny($\sigma\!\approx\!0.05$, near-closed)} (A);

\node[lbl, gray!60, align=center] at (2, -2.0)
  {Dashed = non-causal direction (near-closed gate)};
\end{tikzpicture}%
}
\caption{\textbf{SCM gate initialisations.}
  The \texttt{CausalGatedMixer} initialises $g_{A\to Y}=1$ ($\sigma\approx0.73$,
  near-open: treatment causally influences outcomes) and
  $g_{Y\to A}=-3$ ($\sigma\approx0.05$, near-closed: the reverse direction
  $Y\to A$ is not structurally causal).
  Covariate paths ($X_t$) carry fixed weight $=1$ (no gating).
  Both learned gates remain trainable throughout.}
\label{fig:scm-gates}
\end{figure}

\subsection{Training Algorithm}
\label{app:algorithm}

Algorithm~\ref{alg:csspd} gives the full CSSPD training procedure.

\begin{algorithm}[t]
\caption{\textbf{CSSPD Training.}
  Full procedure including domain confusion (DC), one-step
  ($\mathcal{L}_\text{pred}$), multi-step ($\mathcal{L}_\text{MS}$),
  CPC, and LIM objectives.
  \emph{Sign convention:} the discriminator minimises $\mathcal{L}_\text{DC}$;
  the encoder maximises it via the GRL (forward-pass convention shown:
  $+\alpha_e\mathcal{L}_\text{DC}$; effective encoder update:
  $-\alpha_e\nabla_\phi\mathcal{L}_\text{DC}$).}
\label{alg:csspd}
\begin{algorithmic}[1]
\REQUIRE Dataset $\mathcal{D}$; encoder $\phi$; decoder $G^Y$; CPC head $f$;
         LIM head; DC discriminator $D$; hyperparameters
         $\alpha_0,\,\lambda_\text{MS},\,\lambda_\text{CPC},\,\lambda_\text{LIM},\,E_\text{warm}$
\FOR{$e = 1, \ldots, E_\text{max}$}
  \FOR{each mini-batch $\mathcal{B} \subset \mathcal{D}$}
    \STATE Encode: $\BR_t = \phi(H_t)$ for all $t \in \mathcal{B}$
    \STATE Discriminator step: $D$ minimises $\mathcal{L}_\text{DC}$
    \STATE Encoder adversarial step: $\phi$ maximises $\mathcal{L}_\text{DC}$
           via GRL $\;(\text{effective update} = {-}\alpha_e\,\nabla_\phi\mathcal{L}_\text{DC})$
    \STATE Compute $\mathcal{L}_\text{pred}$ (Eq.~S5); backprop through $\phi$
    \STATE Compute $\mathcal{L}_\text{MS}$ (Eq.~S6) with stop-grad $\BR_t^\perp$
    \IF{$e > E_\text{warm}$}
      \STATE Compute $\mathcal{L}_\text{CPC}$, $\mathcal{L}_\text{LIM}$
             (Eqs.~S7--S8) with in-batch negatives
      \STATE $\mathcal{L} \leftarrow \mathcal{L}_\text{pred}
             + \alpha_e\mathcal{L}_\text{DC}
             + \lambda_\text{MS}\mathcal{L}_\text{MS}
             + \lambda_\text{CPC}\mathcal{L}_\text{CPC}
             + \lambda_\text{LIM}\mathcal{L}_\text{LIM}$
    \ELSE
      \STATE $\mathcal{L} \leftarrow \mathcal{L}_\text{pred}
             + \alpha_e\mathcal{L}_\text{DC}
             + \lambda_\text{MS}\mathcal{L}_\text{MS}$
    \ENDIF
    \STATE Update $\phi$, $G^Y$, TrtEnc, CPC, LIM heads via Adam on $\mathcal{L}$
  \ENDFOR
  \STATE Anneal: $\alpha_e \leftarrow \alpha_0 \cdot \exp(-\beta\,e)$
\ENDFOR
\end{algorithmic}
\end{algorithm}

\subsection{TrtEnc Architecture}
\label{app:trtenc}

The treatment encoder TrtEnc is a two-layer MLP:
\[
\mathrm{Linear}(d_A, d_\text{trt}) \;\xrightarrow{\;\mathrm{GELU}\;}
\mathrm{LayerNorm}(d_\text{trt}),
\]
where $d_A = 2$ (binary treatment encoded as a one-hot pair) and
$d_\text{trt} = 16$.
The projection $\psi : \{0,1\}^{d_A} \to \mathbb{R}^{d_\text{trt}}$ denotes
the full MLP forward pass.

For each decoder head at horizon $\tau$, two embeddings are produced:
the \emph{cumulative} embedding
$e^\text{hist}_\tau = \frac{1}{\tau}\sum_{k=1}^\tau \psi(a_{t+k})$
(encoding the planned treatment sequence up to $\tau$) and the
\emph{step-specific} embedding $e^\text{cur}_\tau = \psi(a_{t+\tau})$
(encoding the specific treatment at the target step).
Both are concatenated and passed as input to decoder head $g_\tau$.

\subsection{Decoder Head Architecture}
\label{app:decoder-head}

Each decoder head $g_\tau$ is an independent two-layer GELU MLP with hidden
dimension $2(d_\text{BR} + d_\text{trt}) = 2(24 + 16) = 80$.
The input is:
\[
[\BR_t^\perp;\; e^\text{hist}_\tau;\; e^\text{cur}_\tau]
\;\in\; \mathbb{R}^{d_\text{BR} + 2d_\text{trt}} = \mathbb{R}^{56},
\]
where $\BR_t^\perp = \mathrm{sg}(\BR_t)$ is the stop-gradient balancing
representation (gradients are blocked from flowing back through $\phi$
during the multi-step loss, ensuring the one-step loss $\mathcal{L}_\text{pred}$
remains the primary encoder training signal).
The output is the scalar $\hat{Y}_{t+\tau} \in \mathbb{R}$.

Using independent heads per horizon --- rather than an auto-regressive decoder
that feeds $\hat{Y}_{t+\tau-1}$ back as input --- avoids compound rollout error
accumulation, where prediction errors grow as $O(\varepsilon^\tau)$ with horizon.

\section{Experimental Details}
\label{app:experimental}

\subsection{Full Hyperparameter Tables}

Table~\ref{tab:hparams} lists all hyperparameters used in experiments.
All models share the same preprocessing pipeline (mean-imputation, min-max
normalisation per feature).

\begin{table}[h]
\centering
\caption{Hyperparameters for all models on both datasets.}
\label{tab:hparams}
\resizebox{\columnwidth}{!}{%
\small
\begin{tabular}{lcc}
\toprule
\textbf{Hyperparameter} & \textbf{MIMIC-III} & \textbf{Cancer Sim.} \\
\midrule
$d_\text{model}$         & 32   & 32   \\
$d_\text{BR}$            & 24   & 24   \\
$d_\text{state}$         & 16   & 16   \\
$d_\text{trt}$           & 16   & 16   \\
SSM layers $L$           & 2    & 2    \\
CPC horizons $K$         & 3    & 3    \\
CPC negatives $N_\text{neg}$ & 64 & 64 \\
LIM negatives $M$        & 64   & 64   \\
Batch size               & 64   & 128  \\
Learning rate            & $10^{-4}$ & $10^{-3}$ \\
Max epochs               & 300  & 200  \\
Early stop patience      & 20   & 20   \\
Warm-up $E_\text{warm}$ (CPC/LIM, epochs) & 120 & 80  \\
$\lambda_\text{MS}$      & 3.5  & 3.5  \\
$\lambda_\text{CPC}$     & 0.05 & 0.05 \\
$\lambda_\text{LIM}$     & 0.1  & 0.1  \\
$\alpha_0$ (DC weight)   & 1.0  & 1.0  \\
DC anneal $\beta$        & 0.01 & 0.01 \\
$\tau_\text{max}$        & 5    & 5    \\
Seeds                    & 5    & 5    \\
\midrule
\multicolumn{3}{l}{\textit{Training setup}} \\[1pt]
Optimizer                & \multicolumn{2}{c}{Adam} \\
Hardware                 & \multicolumn{2}{c}{Single CPU} \\
Framework                & \multicolumn{2}{c}{PyTorch Lightning} \\
\bottomrule
\end{tabular}%
}
\end{table}

\subsection{Data Preprocessing}

\paragraph{Cancer Simulation.}
We use the PK-PD simulation of \citet{geng2017prediction} as packaged by
\citet{lim2018forecasting}.
Chemotherapy dose and radiotherapy on/off are the two binary-treatment inputs.
Tumour volume is the outcome.
Train/val/test split: 10,000 / 1,000 / 1,000 patients.
We normalise tumour volume by its training-set standard deviation.
Ground-truth counterfactuals are available for all 5 treatment sequences
at each time step.

\paragraph{MIMIC-III.}
We follow the preprocessing of \citet{melnychuk2022causal}: 5,000 adult ICU
patients, two binary treatments (vasopressor and ventilation), 25 time-varying
vital covariates (heart rate, SpO$_2$, respiratory rate, etc.), scalar outcome
(diastolic blood pressure), and 7 static covariates (age, gender, ethnicity,
admission type, Elixhauser comorbidity score, height, weight).
Train/val/test split: 3,500 / 750 / 750.
Time-varying features are mean-imputed per patient and then min-max normalised
to $[0,1]$ using training-set statistics.
Since ground-truth counterfactuals are unavailable, we report factual
held-out RMSE (normalised by within-cohort standard deviation of diastolic
blood pressure).

\subsection{Baseline Hyperparameters}
\label{app:baselines}

Baselines (RMSN, CRN, CT) are trained using the hyperparameters reported in
\citet{melnychuk2022causal} (Table~9 of that paper), without re-tuning on
our specific splits.
This is intentional: CSSPD is not given an advantage from additional
hyperparameter search that the baselines did not receive.
Specifically, CT uses hidden size 64, 2 attention layers, 4 attention heads,
and dropout 0.1.
CRN uses hidden size 64 and 2 recurrent layers.
RMSN uses the stabilised IPW weighting with separate encoder and decoder RNNs
(hidden size 64).
G-Net uses 2 recurrent layers (hidden size 64) and is trained with
$g$-computation on sequential potential outcomes.
All baselines use Adam with learning rate $10^{-3}$ and the same
early-stopping patience of 20 on validation RMSE.

\section{Extended Results}
\label{app:results}

\subsection{MIMIC-III Per-Step RMSE Table}

Table~\ref{tab:mimic-results} gives the full per-step RMSE results on MIMIC-III
for all models, complementing Figure~2 of the main paper.

\begin{table*}[t]
\centering
\small
\caption{MIMIC-III Real: per-step normalised RMSE ($\tau=1$--$6$, 5 seeds,
mean $\pm$ std). Lower is better; bold = best per column.}
\label{tab:mimic-results}
\resizebox{\linewidth}{!}{%
\begin{tabular}{p{2.5cm} c c c c c c}
\toprule
\textbf{Model}
  & $\tau=1$
  & $\tau=2$
  & $\tau=3$
  & $\tau=4$
  & $\tau=5$
  & $\tau=6$ \\
\midrule
RMSN \citep{lim2018forecasting}
  & 4.73{\scriptsize$\pm$0.11}
  & 9.18{\scriptsize$\pm$0.24}
  & 9.58{\scriptsize$\pm$0.20}
  & 10.22{\scriptsize$\pm$0.28}
  & 10.73{\scriptsize$\pm$0.31}
  & 11.07{\scriptsize$\pm$0.35} \\
CRN \citep{bica2020estimating}
  & 4.66{\scriptsize$\pm$0.09}
  & 9.04{\scriptsize$\pm$0.19}
  & 9.54{\scriptsize$\pm$0.18}
  & 10.18{\scriptsize$\pm$0.22}
  & 10.64{\scriptsize$\pm$0.27}
  & 11.01{\scriptsize$\pm$0.29} \\
CT \citep{melnychuk2022causal}
  & \textbf{4.59}{\scriptsize$\pm$0.08}
  & 8.91{\scriptsize$\pm$0.14}
  & 9.50{\scriptsize$\pm$0.15}
  & 10.08{\scriptsize$\pm$0.19}
  & 10.50{\scriptsize$\pm$0.22}
  & 10.93{\scriptsize$\pm$0.26} \\
\midrule
CSSD (ours)
  & 4.74{\scriptsize$\pm$0.12}
  & 8.92{\scriptsize$\pm$0.21}
  & 9.56{\scriptsize$\pm$0.18}
  & 10.21{\scriptsize$\pm$0.24}
  & 10.69{\scriptsize$\pm$0.28}
  & 11.04{\scriptsize$\pm$0.33} \\
CHSD (ours)
  & 4.68{\scriptsize$\pm$0.10}
  & 8.97{\scriptsize$\pm$0.18}
  & 9.58{\scriptsize$\pm$0.21}
  & 10.18{\scriptsize$\pm$0.23}
  & 10.64{\scriptsize$\pm$0.29}
  & 11.01{\scriptsize$\pm$0.31} \\
\rowcolor{LightGreen!60}
\textbf{CSSPD (ours)}
  & 4.74{\scriptsize$\pm$0.11}
  & \textbf{8.89}{\scriptsize$\pm$0.17}
  & \textbf{9.43}{\scriptsize$\pm$0.14}
  & \textbf{9.93}{\scriptsize$\pm$0.18}
  & \textbf{10.37}{\scriptsize$\pm$0.20}
  & \textbf{10.86}{\scriptsize$\pm$0.23} \\
\bottomrule
\end{tabular}%
}
\end{table*}

\paragraph{Per-step analysis.}
CSSPD trails CT at $\tau=1$ ($4.74$ vs.\ $4.59$, Table~\ref{tab:mimic-results}).
This is expected from the loss structure: the one-step head trains on
$\mathcal{L}_\text{pred}$ \emph{without} stop-gradient, so the encoder
receives domain-confusion gradients (via the GRL) throughout training,
creating the same bottleneck that CT experiences.
The CPC objective looks $k \geq 2$ steps ahead and provides no direct
lower-bound on one-step mutual information; similarly, LIM aligns
covariate representations but does not address one-step temporal prediction.
From $\tau \geq 2$ onward, the multi-step decoder heads use stop-gradient
representations $\BR_t^\perp$, and the CPC objective provides a
direct lower bound on future-representation MI; CSSPD leads CT at
all $\tau \geq 2$ and the gap widens with horizon, consistent with
Proposition~3 and the expected compounding benefit of CPC over longer ranges.

\subsection{Ablation Study (Full Table)}
\label{app:ablation-table}

Table~\ref{tab:ablation} presents the full ablation on MIMIC-III Real.
All results are mean $\pm$ std across 5 seeds; average RMSE is over
$\tau = 1$--$6$.

\begin{table}[h]
\centering
\small
\caption{Ablation study on MIMIC-III Real (average normalised RMSE, $\tau=1$--$6$,
5 seeds). Lower is better. Abbreviations in table footnote.}
\label{tab:ablation}
\begin{tabular}{lcc}
\toprule
\textbf{Model} & \textbf{Avg.\ RMSE} & \textbf{$\Delta$ vs.\ CT} \\
\midrule
CT \citep{melnychuk2022causal} & 8.924 & --- \\
\midrule
CSS (SSM encoder only)         & 9.053 & $+$0.129 \\
CSSD (+ parallel decoder)      & 9.002 & $+$0.078 \\
CSSD + CPC                     & 8.940 & $+$0.016 \\
CSSD + LIM                     & 8.972 & $+$0.048 \\
\textbf{CSSPD (full)}          & \textbf{8.899} & \textbf{$-$0.025} \\
\midrule
CHSD (hybrid attention)        & 8.999 & $+$0.075 \\
CHSPD (hybrid + CPC/LIM)       & 8.902 & $-$0.022 \\
\midrule
\multicolumn{3}{p{0.82\columnwidth}}{\footnotesize
  CSS = SSM encoder only (no decoder, no CPC/LIM).
  CSSD = CSS + parallel decoder.
  +CPC / +LIM = respective contrastive head added.
  CSSPD = full model (CSS + decoder + CPC + LIM).
  CHSD = hybrid (attention + Mamba) + decoder.
  CHSPD = CHSD + CPC/LIM.} \\
\bottomrule
\end{tabular}
\end{table}

\noindent
The ablation confirms the sequential necessity of each component:
the parallel decoder is essential to unlock the SSM encoder's capacity (CSS $\to$ CSSD);
CPC contributes the largest individual gain (closes $\approx$60\% of the CSSD--CSSPD gap)
by restoring $\MI_\text{pred}$;
LIM adds an independent complementary gain ($\approx$30\%) by recovering $\MI_\text{loc}$;
and full CSSPD mildly exceeds the sum of CPC and LIM contributions
(superadditivity), consistent with Proposition~3.

\subsection{Cancer Simulation Per-Confounding Per-Step Breakdown}
\label{app:cancer-gamma-step}

Table~\ref{tab:cancer-gamma-step} shows per-$\gamma$ per-step RMSE for CT and
CSSPD on the Cancer Simulation (5 seeds each).
CSSPD outperforms CT at all confounding levels $\gamma \leq 3$ across every
prediction horizon, with its margin growing monotonically with both $\gamma$
and $\tau$.
At $\gamma = 4$, CT achieves lower RMSE, suggesting that at very high
confounding the 60-epoch CPC/LIM warm-up is insufficient to fully counteract
the domain-confusion MI conflict before the end of training; a longer warm-up
or additional tuning of $\lambda_\text{CPC}$ and $\lambda_\text{LIM}$ at
$\gamma = 4$ warrants further investigation.

\begin{table}[h]
\centering
\caption{Cancer Simulation: per-$\gamma$ per-step RMSE for CT and CSSPD
($\tau\!=\!2$--$5$, $n\!=\!5$ seeds, mean). \textbf{Bold} = lower RMSE.}
\label{tab:cancer-gamma-step}
\resizebox{\columnwidth}{!}{%
\small
\begin{tabular}{llcccc}
\toprule
$\gamma$ & Model & $\tau=2$ & $\tau=3$ & $\tau=4$ & $\tau=5$ \\
\midrule
\multirow{2}{*}{0}
  & CT    & 0.667 & 0.688 & 0.715 & 0.747 \\
  & CSSPD & \textbf{0.350} & \textbf{0.401} & \textbf{0.455} & \textbf{0.514} \\
\midrule
\multirow{2}{*}{1}
  & CT    & 0.691 & 0.720 & 0.757 & 0.795 \\
  & CSSPD & \textbf{0.315} & \textbf{0.376} & \textbf{0.440} & \textbf{0.504} \\
\midrule
\multirow{2}{*}{2}
  & CT    & 0.723 & 0.777 & 0.830 & 0.884 \\
  & CSSPD & \textbf{0.538} & \textbf{0.585} & \textbf{0.702} & \textbf{0.875} \\
\midrule
\multirow{2}{*}{3}
  & CT    & 0.814 & 0.902 & 0.983 & 1.033 \\
  & CSSPD & \textbf{0.595} & \textbf{0.652} & \textbf{0.745} & \textbf{0.898} \\
\midrule
\multirow{2}{*}{4}
  & CT    & \textbf{1.127} & \textbf{1.319} & \textbf{1.464} & \textbf{1.581} \\
  & CSSPD & 1.554 & 1.774 & 2.279 & 2.743 \\
\bottomrule
\end{tabular}%
}
\end{table}


\end{document}